\documentclass[opre,nonblindrev]{informs3_homepage}

\OneAndAHalfSpacedXI

\usepackage{microtype}
\usepackage{graphicx}
\usepackage{booktabs}
\usepackage{xcolor}
\usepackage{array}
\usepackage[hidelinks]{hyperref}
\usepackage{suffix}

\usepackage{caption}
\usepackage{subcaption}

\usepackage{algorithm}
\usepackage{algorithmic}
\usepackage{tikz}
\usetikzlibrary{calc,backgrounds,arrows.meta}

\usepackage{enumitem}
\usepackage{textgreek}

\usepackage[informs]{autoconfig}

\DeclareMathAlphabet\mathbfcal{OMS}{cmsy}{b}{n}

\def\Policy{\pi}
\def\PolicyTS{\Policy^{\operatorname{TS}}}
\def\PolicyROFUL{\Policy^{\operatorname{ROFUL}}}
\def\PolicyOFUL{\Policy^{\operatorname{OFUL}}}
\def\PolicySG{\Policy^{\operatorname{SG}}}

\def\Arm{A}
\def\PolicySet{\mathsf{P}}
\def\ArmDomain{\mathbfcal{\Arm}}
\def\ArmSet{\mathcal{\Arm}}

\def\ArmBound{\mathbf{a}}
\def\ArmNum{k}
\def\ChosenArm{\widetilde{\Arm}}
\def\OptimalArm{\Arm^\star}
\def\Reward{\mathcal{R}}
\def\Regret{\operatorname{Regret}}
\def\History{\mathcal{F}}
\def\HistoryPlus{\History}
\def\Uncertainty{\mathsf{V}}
\def\Complexity{\mathsf{K}}
\def\Param{\Theta^\star}
\def\Entropy{\mathsf{H}}
\def\UpperConf{\mathsf{U}}
\def\LowerConf{\mathsf{L}}
\def\GainRate{\mathsf{G}}
\WithSuffix\def\GainRate*[#1]{\GainRate_{#1}}
\def\GapLevel{\delta}
\def\GapProb{\mathsf{q}}
\WithSuffix\def\GapProb*[#1]{\GapProb_{#1}}
\def\Deviation{\mathsf{D}}
\WithSuffix\def\Deviation*[#1]{\Deviation_{#1}}
\def\Gap{\mathsf{\Delta}}
\def\IGap{\mathbb{G}}
\def\BigRegret{\mathbb{B}}
\def\OptProb{\mathsf{p}}
\def\OptCoef{\omega}
\def\BaseLevel{\mathsf{B}}
\def\MeanReward{\mathsf{M}}

\def\Worth{\widetilde{\MeanReward}}
\def\ITypical{\mathbb{T}}
\def\ITypicalM{\ITypical^\star}
\def\ITypicalW{\widetilde{\ITypical}}
\def\eps{\varepsilon}
\def\semdefleq{\preccurlyeq}
\def\Inflation{\iota}
\def\TsSample{\widetilde{\Theta}}
\def\ParamBound{\boldsymbol{\theta}}
\def\SieveRate{\alpha}
\def\ParamTail{\xi}
\def\DivMatrix{\Upsilon}
\def\DivParam{\upsilon}
\def\DivCovOp{\psi}
\def\NearOptimal{\mathbf{W}}
\def\GroupSpace{\mathbf{Z}}

\def\Orthsum{\oplus}
\def\LinExpansion{\mathbb{L}}
\def\IOptim{\mathbb{O}}

\def\GeneralMargin{\tau}

\def\SigmaAlgebra{\text{\textsigma}}

\def\Eye{\mathbf{I}}
\def\Dim{d}
\def\CovMatrix{\mathbf{\Gamma}}
\def\SymCovMatrix{\mathbf{\Sigma}}
\def\Estimator{\widehat{\Theta}}
\def\Radius{\rho}
\def\GroupedRadius{\eta}
\def\Extreme{\mathsf{E}}
\def\Profile{\mathsf{\Pi}}

\def\BernoulliFamily{\mathcal{B}}
\def\Iset{\mathcal{I}}

\def\II{\mathbb{I}}

\usepackage{natbib}
 \bibpunct[, ]{(}{)}{,}{a}{}{,}%
 \def\bibfont{\small}%
\TheoremsNumberedThrough     
\ECRepeatTheorems

\EquationsNumberedBySection 

\MANUSCRIPTNO{MS-0001-1922.65}

\begin{document}



\RUNTITLE{
General Framework for Linear Bandit
}

\TITLE{
A General Theory of the Stochastic Linear Bandit and Its Applications
}


\ARTICLEAUTHORS{%
	\AUTHOR{Nima Hamidi}
	\AFF{Department of Statistics, Stanford University, \EMAIL{hamidi@stanford.edu}}
	\AUTHOR{Mohsen Bayati}
	\AFF{
		Graduate School of Business, Stanford University, \EMAIL{bayati@stanford.edu}}
	} 

\ABSTRACT{%
Recent growing adoption of experimentation in practice has led to a surge of attention to multiarmed bandits as a technique to reduce the opportunity cost of online experiments. In this setting, a decision-maker sequentially chooses among a set of given actions,  observes their noisy rewards, and aims to maximize her cumulative expected reward (or minimize regret) over a horizon of length $T$.  
In this paper, we introduce a general analysis framework and a family of algorithms for the stochastic linear bandit problem that includes well-known algorithms such as the \emph{optimism-in-the-face-of-uncertainty-linear-bandit} (OFUL) and \emph{Thompson sampling} (TS) as special cases.
Our analysis technique bridges several streams of prior literature and yields a number of new results. First, our new notion of \emph{optimism in expectation} gives rise to a new algorithm, called \emph{sieved greedy} (SG) that reduces the overexploration problem in OFUL. SG utilizes the data to discard actions with relatively low uncertainty and then choosing one among the remaining actions greedily. In addition to proving that SG is theoretically rate optimal, our empirical simulations show that SG outperforms existing benchmarks such as greedy, OFUL, and TS.
The second application of our general framework is (to the best of our knowledge) the first polylogarithmic (in $T$) regret bounds for OFUL and TS, under similar conditions as the ones by \cite{goldenshluger2013linear}.  Finally, we obtain sharper regret bounds for the $k$-armed contextual MABs by a factor of $\sqrt{k}$.
}%


\KEYWORDS{Stochastic Linear Bandit, Contextual Bandit, Thompson Sampling, Optimism in the Face of Uncertainty, Greedy Algorithms}

\maketitle
%
\section{Introduction}
\label{sec:intro}

Recently, multiarmed bandit (MAB) experiments
have received extensive attention due to their potential for reducing the opportunity cost of running online experiments \citep{scott2010modern,scott2015multi,johari2017peeking}. Specifically,  MAB experiments allow adaptive adjustments to the design of the experiments, based on partially available data during the experiment. The MAB approach was first motivated by the cost of experimentation in clinical trials \citep{thompson1933likelihood,lai1985asymptotically}.

More formally, in a MAB problem, a decision-maker, also known as the \emph{policy} or \emph{algorithm}, sequentially chooses actions from given action sets and receives rewards corresponding to the selected actions. The goal is to maximize the cumulative reward throughout the experimentation periods, by utilizing the \emph{history} of previous observations.  Alternatively, the aim is to choose a policy that minimizes the cumulative \emph{regret}, which is the difference between the highest achievable reward by a clairvoyant decision-maker who knows the expected reward of each action relative to the reward obtained by the policy.
This paper considers a variant of this problem, called \emph{stochastic linear bandit}, in which all actions are elements of $\IR^d$ for a positive integer $d$ and the expected value of the reward depends on the actions via a linear function. This class of problems includes the well-known subclass of $k$-armed contextual MABs as a special case, when the action sets are allowed to be time-dependent.

Since its introduction by \cite{naoki1999associative}, the linear bandit problem has attracted a great deal of attention. Several algorithms based on the idea of optimism or upper confidence bound (UCB), due to \cite{lai1985asymptotically}, have been proposed and analyzed. Notable examples are by \cite{auer2003using,dani2008stochastic,rusmevichientong2010linearly,abbasi2011improved}, and \cite{lattimore2015pareto}. The best algorithm in this class is the optimism-in-the-face-of-uncertainty-linear-bandit (OFUL) algorithm of \cite{abbasi2011improved} with the regret $\Order[\big]{d\sqrt{T}\log^{3/2}T}$ that matches the best lower bound due to \cite{dani2008stochastic} up to logarithmic factors.

A second line of research examines the performance of Thompson sampling (TS) or posterior sampling, a Bayesian heuristic due to \cite{thompson1933likelihood} that employs the posterior distribution of the reward function to balance exploration and exploitation and reduce regret.
\cite{russo2014learning,dong2018information} proved an $\Order[\big]{d\sqrt{T}\log T}$ upper bound for the \emph{Bayesian} regret of TS, thereby indicating its near-optimality. 

In addition, when there is a deterministic gap $\delta$ between the expected rewards of the top two actions, OFUL and TS are shown to have a regret with a $\text{poly}(\log T)/\delta$ dependence in $T$ instead of a $\sqrt{T}\log T$ one. But this bound is not applicable when $\GapLevel$ is exactly zero.  In fact, this happens for the well-known subclass of linear $k$-armed contextual MABs.  Hence, the prior $\log T/\GapLevel$ bounds for OFUL or TS are not applicable. One needs a more general (probabilistic) notion of the gap to study these types of problems.  This is in fact the subject of the third stream of research, pioneered by \cite{goldenshluger2013linear}, that leverages a so-called \emph{margin condition} to model probabilistic reward $\GapLevel$. They showed that the best lower bound for the contextual MAB is logarithmic in $T$, and proposed a variant of the $\epsilon$-greedy algorithm that achieves this bound. This idea was extended by \cite{bastani2020online} to settings where contexts are high-dimensional (i.e., $d$ becomes very large). However, both of these papers propose algorithms that require an input parameter $h$ to adjust for the probabilistic gap.
\emph{It is an open problem whether such logarithmic (in $T$) bounds for OFUL and TS, that do not take any gap parameter as input can be proved, under the same conditions as in \citep{goldenshluger2013linear}}. In addition, while the first two streams of the aforementioned research were mostly united by the results of \cite{russo2014learning} and \cite{abeille2017linear} that connected OFUL and TS, there was a disconnect between them and the third stream of research.

\paragraph{Contributions.} In this paper we propose an analysis framework for the stochastic linear bandit problem that bridges all three aforementioned streams of literature and yields a number of new results. To be explicit, the main contributions of this paper are as follows:
\begin{enumerate}
	\item  We propose a general family of algorithms, called randomized OFUL (ROFUL), for the stochastic linear bandit problem and prove that they are rate optimal (their regret matches the best lower bound). We also show that OFUL and TS are special cases of this family of algorithms and that our regret bound for ROFUL recovers previously known rate-optimal regret bounds for OFUL and TS in Bayesian and frequentist settings, respectively.
	
	\item Most importantly, we employ the margin assumption of \cite{goldenshluger2013linear} to model a probabilistic gap that allows $\GapLevel$ to be zero, in order to obtain a polylogarithmic (in $T$) regret bound for OFUL and TS when the action sets are independently drawn from an unknown distribution; this includes the contextual MAB problem as in \citep{goldenshluger2013linear}. To the best of our knowledge, this result had not been known for OFUL and TS.

	\item Our analysis of ROFUL naturally leads us to introduce a new rate-optimal policy, \emph{Sieved Greedy} (SG), that leverages data to
	reduce the overexploration problem in OFUL and TS. A key technique to achieve this is to employ a more general form of the optimism principle that we introduce, called \emph{optimism in expectation}, that 
allows data-driven exploration by focusing only on actions with sufficient uncertainty, and then choosing one of them greedily. Our empirical simulations show that SG outperforms greedy, OFUL, and TS.
	
	While SG has the same spirit as recent literature on algorithms that put more emphasis on greedy decision-making \citep{bastani2017mostly,kannan2018smoothed,raghavan2018externalities,hao2019adaptive,bayati2020unreasonable}, it relies on the notion of optimism in expectation that is fundamentally a different idea compared to the ones powering the prior literature (e.g., covariate diversity or large number of arms). Investigating potential connections between all these algorithms is a tantalizing direction for future research.
	
	\item Motivated by the fact that the $k$-armed $d$-dimensional contextual MAB problem is a special case of the $kd$-dimensional stochastic linear bandit, see \cite{abbasi-yadkori2012online} for the reduction, we formulate a slightly more general version of the stochastic linear bandit that we refer to as the \emph{grouped linear bandit} (GLB). Then, using the structure of GLB, we obtain sharper regret bounds (by a factor $\sqrt{k}$) for our general ROFUL algorithm. Therefore, we obtain sharper regret bounds for OFUL and TS than those that can be obtained by directly applying the existing approaches studied by \cite{abbasi2011improved,russo2014learning,agrawal2013thompson}, and \cite{abeille2017linear}.
	
\end{enumerate}

\subsection{Other literature}

Some of the components in our analysis of the ROFUL algorithm have similarities with prior literature \citep{srinivas2010gaussian,russo2016information,kirschner2018heteroscedastic}. Specifically, our notion of uncertainty complexity is similar to the notion of maximum information by \cite{kirschner2018heteroscedastic}. We discuss the differences between the two in \cref{rem:uncertainty-similar-to-KR18}, but in summary our approach provides regret bounds for the more general probabilistic $\GapLevel$ setting as well. In addition, our notion of gain rate is similar to the notion of information ratio by \cite{russo2016information}. We discuss the differences between the two in \cref{rem:gain-rate-similar-to-RV16}, but in summary  \cite{russo2016information} consider a Bayesian setting while we consider both Bayesian and frequentist settings. 

\subsection{Organization}

We introduce notation and the problem formulation in \cref{sec:setting}. Then we introduce
uncertainty complexity and its connection to regret in Sections \ref{sec:uncertainty} and \ref{sec:regret}. Our ROFUL algorithm and its regret analysis are presented in \cref{sec:optimism}. In \cref{sec:examples}, we first demonstrate how OFUL and TS are special cases of ROFUL and then introduce our rate-optimal SG algorithm, which is empirically compared with existing benchmarks in \cref{sec:simulation}. Finally, in \cref{sec:margin}, we provide extensions of our results to obtain polylogarithmic regret bounds for OFUL and TS and sharper bounds for the $k$-armed contextual bandit problem. 

\section{Problem Formulation and Notation}
\label{sec:setting}
For any positive integer $n$, we denote $\{1,2,\cdots,n\}$ by $[n]$. The transpose of a vector $A$ is denoted by $A^\top$. For positive semidefinite matrix
$\SymCovMatrix\in\IR^{n\times n}$ and for any vector $\Arm\in\IR^n$,
notation  $\Norm{\Arm}_{\SymCovMatrix}$ refers to $\sqrt{\Arm^\top\SymCovMatrix \Arm}$. We also use  notation $\Eye_\Dim$ to denote the $d$-by-$d$ identity matrix.

Let $\ArmDomain$ be the set of all possible actions and let $(\ArmSet_t)_{t=1}^T$ be a sequence of $T$ random subsets of $\ArmDomain$, where $T\in\IN$ will be referred to as the time horizon. A policy $\Policy$ sequentially interacts with this environment in $T$ rounds. At time $t\in[T]$, the action set $\ArmSet_t$ is revealed to the policy and it chooses an action $\ChosenArm_t\in\ArmSet_t$ and receives a stochastic reward $\Reward(\ChosenArm_t)$.
We also assume that $\ArmDomain\subset\IR^d$ and is bounded; that is,
there exists a positive constant $\ArmBound$ such that $\NormTwo{\Arm}\leq\ArmBound$ for all $\Arm\in\ArmDomain$. Moreover, we assume there exists a random vector $\Param\in\IR^d$ for which
\begin{equation}\label{eq:linearity}
\Expect[]{\Reward(\Arm)\Given\Param}=\Inner{\Param,\Arm}\,,
\end{equation}
for all $\Arm\in\ArmSet_t$
where $\Inner[\big]{\cdot,\cdot}$ is the standard dotproduct on $\IR^d$. We also assume that there exists a positive parameter $\ParamBound$ such that distribution of $\Param$ satisfies
\begin{align}
	\forall \ParamTail>0~,~~~\Prob{\Norm{\Param}_2\geq\ParamBound+\ParamTail}\leq e^{-\ParamTail^2/2}\,.\label{eq:param-bound}
\end{align}
For example, if $\Param$ is bounded then \cref{eq:param-bound} easily holds. Another important setting where \cref{eq:param-bound} holds is when $\Param$ has a normal distribution.

Next, we introduce the notation
\[
\MeanReward_t(\Arm):=\Inner[\big]{\Param, \Arm}\,,
\]
and assume that there exists (random) optimal action $\OptimalArm_t\in\ArmSet_t$ such that the following holds almost surely for all $\Arm\in\ArmSet_t$,
\begin{align}
    \MeanReward_t(\OptimalArm_t)
    \geq
    \MeanReward_t(\Arm)\,.\label{eq:optimal-reward}
\end{align}
Now, consider the sequence of $\SigmaAlgebra$-algebras $\History_0\subseteq\History_1\cdots\subseteq\History_{t-1}$ that encode history of observations up to time $t$ and are defined by
\begin{align*}
    \History_{0}:=\SigmaAlgebra(\emptyset)
    ~~~~\text{and}~~~~
    \History_{t-1}:=\SigmaAlgebra(\ArmSet_1,\ChosenArm_1,\Reward(\ChosenArm_1),\ldots,\ArmSet_{t-1},\ChosenArm_{t-1},\Reward(\ChosenArm_{t-1}),\ArmSet_t)\,.
\end{align*}
%
In this model, a \emph{policy} $\Policy$ is formally defined as a deterministic function that maps $\HistoryPlus_{t-1}$ to an element of $\ArmSet_t$.

Moreover, for each chosen action $\ChosenArm_t\in\ArmSet_t$, its stochastic reward $\Reward(\ChosenArm_t)$ is equal to $\Inner[\big]{\Param, \ChosenArm_t}+\eps_t$ where, conditional on $\History_{t-1}$, the noise random variable $\eps_t$ has mean zero and is $\sigma^2$-sub-Gaussian. Specifically, $\Expect*{\,|\eps_t|\Given\History_{t-1}}<\infty$ and
\[
\Expect*{\exp\left(\kappa\,\eps_t\right)\Given\History_{t-1}}\leq  \exp\left(\frac{\sigma^2\kappa^2}{2}\right)\,,~~~\forall~\kappa\in\IR\,.
\]

The performance measure for evaluating the policies is the standard cumulative Bayesian regret defined as
\begin{align*}
    \Regret(T,\pi):=\sum_{t=1}^{T}\Expect*{\sup_{A\in\ArmSet_t}\MeanReward_t(A)-\MeanReward_t(\ChosenArm_t)}\,.
\end{align*}
The expectation is taken with respect to the entire randomness in our model, including the prior distribution of $\Param$.
Although we assumed that $\Param$ is random, and we described a  Bayesian formulation of regret, our model and results include the deterministic $\Param$ setting as well. This can be achieved by considering the prior distribution to be the distribution with a point mass at $\Param$.

\paragraph{Action sets.}
%
Action sets and their structure play key roles in this paper which require introducing a number of important notions associated with them.
We start by defining the \emph{extremal points} of an action set.
\begin{defn}[Extremal points]
For an action set $\ArmSet$, define its extremal points $\Extreme(\ArmSet)$ to be all $\Arm'\in\ArmSet$ that are not a convex combination of other actions in $\ArmSet$, i.e., actions in $\ArmSet$ for which one cannot find actions $\{\Arm_1,\cdots,\Arm_n\}\subseteq\ArmSet\setminus\{\Arm'\}$ and coefficients $\{c_1,\cdots,c_n\}\subset[0,1]$ satisfying
\begin{align*}
    \Arm'=\sum_{i=1}^nc_i\Arm_i
    ~~~~\text{and}~~~~
    \sum_{i=1}^nc_i=1.
\end{align*}
\end{defn}
The importance of this definition is that all the algorithms studied in this paper choose only extremal points in action sets, because of the linearity assumption on the mean reward as stated in \cref{eq:linearity}. This observation implies that the rewards attained by any of these algorithms, when provided with the action set $\ArmSet$, belong to the \emph{reward profile} of $\ArmSet$ defined by
\begin{align*}
    \Profile_\ArmSet:=\left\{\MeanReward_t(A):\Arm\in\Extreme_\ArmSet\right\}.
\end{align*}
Recall from \cref{eq:optimal-reward} that $\MeanReward_t(\OptimalArm_t)$ is the maximum attainable reward of an action set $\ArmSet$. Building on this, we define \emph{gap} of an action set $\ArmSet$ as
\begin{align*}
    \Gap_\ArmSet:=\MeanReward_t(\OptimalArm_t)-\sup\left(\Profile_\ArmSet\setminus\{\MeanReward_t(\OptimalArm_t)\}\right)\,.
\end{align*}
Moreover, for any $z\geq0$, we define
\begin{align*}
    \ArmSet_t^z:=\left\{\Arm\in\ArmSet_t: \MeanReward_t(\Arm)\geq\MeanReward_t(\OptimalArm_t)-z \right\}\,.
\end{align*}
In the sequel, we may simplify the above notation and use subscript $t$ instead of $\ArmSet_t$. For instance, $\Gap_t$ refers to $\Gap_{\ArmSet_t}$. We now define a \emph{gapped} problem as follows:
\begin{defn}[Gapped problem]
\label{def:gapped-problem}
We call a linear bandit problem \emph{gapped} if for some positive numbers $\GapLevel$ and $    \GapProb*[\GapLevel]$, the following inequality holds:
\begin{align}
    \Prob{\Gap_t\leq\GapLevel}
    \leq
    \GapProb*[\GapLevel]
    ~~~~~~
    \text{for all $t\in[T]$\,,}
    \label{eq:gap-def}
\end{align}
where the probability is calculated with respect to the randomness of the action sets.
Moreover, for a fixed gap level $\delta$, we define $\IGap_t$ to be the indicator of the event $\{\Gap_t\geq\GapLevel\}$.
\end{defn}
\begin{rem}
	The above notion of gap is more general than the well-known notion of gap in the literature, as in  \citep{abbasi2011improved}, which is a deterministic concept. Specifically, we do not assume that the probability $\GapProb*[\GapLevel]$ is equal to $0$, for a fixed $\delta>0$.
\end{rem}
\begin{rem}
All problems are gapped for all $\GapLevel>0$ and $\GapProb*[\GapLevel]=1$ since \cref{eq:gap-def} will be trivially satisfied. This observation will help us obtain gap-independent bounds.
\end{rem}

\section{Uncertainty Complexity}
\label{sec:uncertainty}

In this section we introduce the notion of \emph{uncertainty structure}, which will be a key parameter in obtaining regret bounds in subsequent sections. We also calculate this parameter in three examples to help build intuition.

By uncertainty structure, we simply refer to a sequence of functions $\Uncertainty_t:(\HistoryPlus_{t-1},\Arm)\mapsto\IR$, where $\Arm\in\ArmSet_t$. By a slight abuse of notation, for any policy $\Policy$, we define \emph{expected uncertainty} to be
\begin{align*}
    \Uncertainty\left(\Policy\right)
    :=
    \Expect*{\sum_{t=1}^T\Uncertainty_t(\ChosenArm_t)}\,.
\end{align*}
Finally, for a set of policies $\PolicySet$, the \emph{uncertainty complexity} is defined as
\begin{align*}
    \Complexity
    :=
    \sup_{\Policy\in\PolicySet}\Uncertainty(\Policy)\,.
\end{align*}
Note that uncertainty complexity \emph{is not a unique quantity} for a given problem as the choice of functions $\Uncertainty_t$ can vary. We will see in the following sections that any uncertainty structure, together with an associated \emph{gain rate} that is defined in \cref{sec:regret}, can be used to provide an upper bound for the regret of any policy.  However, the quality of the regret bound does depend on the choice of uncertainty structure.

In order to get a better of sense of uncertainty complexity, in the remainder of this section we provide upper bounds for the uncertainty complexity of several well-known problems. We then use these bounds in \cref{sec:examples} to derive rate-optimal regret bounds for OFUL and variants of TS.
Overall,  the optimal selection of an uncertainty structure is an interesting and challenging research question, but one that is well beyond the scope of this paper.
\begin{rem}\label{rem:uncertainty-similar-to-KR18} The above notion of uncertainty complexity is similar to the notion of \emph{maximum information} by \cite{kirschner2018heteroscedastic}, see their Eq. (2). The main difference is that we do not require the essential supremum of $\sum_t\Uncertainty_t$ to exist. This makes our analysis simpler; see, e.g., the second paragraph on page 6 of \citep{kirschner2018heteroscedastic}. Also, in contrast to \cite{kirschner2018heteroscedastic}, our proof technique provides regret bounds for the (generalized) gapped version of the problem as well.
\end{rem}

%
\begin{exmp}[Unstructured linear bandit]
\label{exa:unstr-lin-bandit}
Let $\lambda$ be a positive and fixed real number and, for any $t\in[T]$, define
\begin{align}
    \SymCovMatrix_t
    :=
    \left(\frac1\lambda\Eye_\Dim+\frac1{\sigma^2}\sum_{s=1}^{t}\ChosenArm_s\ChosenArm_s^\top\right)^{-1}.
    \label{eq:sym-cov-matrix-def}
\end{align}
Then, we choose the following uncertainty structure:
\begin{align*}
    \Uncertainty_t(\Arm)
    :=
    \min\left\{\sigma^2,\Norm{\Arm}_{\SymCovMatrix_{t-1}}^2\right\}.
\end{align*}
Lemmas 10 and 11 of \cite{abbasi2011improved} essentially prove that
\begin{align}
    \Complexity
    \leq
    2\,\sigma^2 d\log\left(1+\frac{T\ArmBound^2\lambda}{\Dim\sigma^2}\right).
    \label{eq:unstr-lin-bandit-complexity-bound}
\end{align}
\end{exmp}
\begin{exmp}[Bayesian linear bandit with fixed finite action sets]
\label{exa:bayes-lin-bandit-finite}
Consider a finite action set $\ArmDomain=\{\Arm_1,\Arm_2,\cdots,\Arm_\ArmNum\}$ that does not change over time. In other words,  $\ArmSet_t=\ArmDomain$ for all $t\in[T]$ almost surely. Following a similar notation as  \cite{russo2016information}, for all $j\in[k]$, we let
\begin{align*}
    \alpha_{t,j}:=\Prob{\OptimalArm=\Arm_j\Given\History_{t-1}}
\end{align*}
and
\begin{align*}
    \mu_{t,j}
    :=
    \Expect{\Param\Given\History_t,\OptimalArm=\Arm_j}
    =
    \frac{1}{\alpha_{t,j}}\Expect{\Param\cdot\II(\OptimalArm=\Arm_j)\Given\History_{t-1}}\,.
\end{align*}
Now, defining $
    \mu_{t}
    :=
    \Expect{\Param\Given\History_{t-1}}$, we consider the following uncertainty functions:
\begin{align*}
    \Uncertainty_t(\Arm)
    :=
    \Norm{\Arm}_{\CovMatrix_{t}}^2\,,
\end{align*}
where
\begin{align}
    \CovMatrix_t
    :=
    \sum_{j=1}^{k}\alpha_j(\mu_{t,j}-\mu_{t})(\mu_{t,j}-\mu_{t})^\top.
    \label{eq:cov-matrix-def-finite}
\end{align}
The analysis of \cite{russo2016information} implies that 
\begin{align}
    \Complexity
    \leq
    2\sigma^2\Entropy(\OptimalArm)\,,
    \label{eq:bayesian-finite-complexity}
\end{align}
where $\Entropy(\OptimalArm)$ is the entropy of $\OptimalArm$. For completeness, we provide a slightly modified version of their proof in \cref{sec:auxi}.
\end{exmp}
\begin{exmp}[Bayesian linear bandit with normal prior and noise]
\label{exa:bayes-lin-bandit-normal}
In this example, we focus on the Bayesian setting in which $\Param\sim\Normal{0,\lambda\Eye_d}$, and at round $t$, the reward of selecting action $\ChosenArm_t$ is given by $\Reward(\ChosenArm_t)=\Inner{\Param,\ChosenArm_t}+\eps_t$ where $\eps_t\sim\Normal{0,\sigma^2}$ is independent of $\History_{t-1}$. However, we allow the action sets to change over time and also allow the action sets to have infinite size. Inspired by the previous example, we define
\begin{align*}
    \Uncertainty_t(\Arm)
    :=
    \Norm{\Arm}_{\CovMatrix_{t}}^2,
\end{align*}
where
\begin{align}
    \CovMatrix_t
    :=
    \Cov[\big]{\Expect{\Param\Given\HistoryPlus_{t-1},\OptimalArm_t}\Given\HistoryPlus_{t-1}}\,.
    \label{eq:cov-matrix-def-infinite}
\end{align}
It is easy to see that in the setting of \cref{exa:bayes-lin-bandit-finite}, the above definition is equivalent to \cref{eq:cov-matrix-def-finite}. We now use a different technique to bound the uncertainty complexity. Notice that the normality assumption yields
\begin{align*}
    \CovMatrix_t
    \semdefleq
    \Cov[\big]{\Param\Given\HistoryPlus_{t-1}}
    =
    \SymCovMatrix_t.
\end{align*}
Therefore, \cref{exa:unstr-lin-bandit} implies that
\begin{align}
    \Complexity
    \leq
       2 \sigma^2 d\log\left(1+\frac{T\ArmBound^2\lambda}{\Dim\sigma^2}\right).
    \label{eq:bayesian-normal-complexity}
\end{align}
\end{exmp}

\section{Regret Bound and Gain Rate}
\label{sec:regret}
In this section, building on the notion of uncertainty structure, we introduce the notion of gain rate of any policy and then use that to obtain an upper bound for the regret.
\begin{defn}[Gain rate]
Let $\GapLevel>0$ be fixed. We say that a policy $\Policy$ has gain rate $\GainRate*[\GapLevel]>0$ with respect to an uncertainty structure $\{\Uncertainty_t\}_{t\ge1}$ if
\begin{align}
    \Expect*{
        \left(\MeanReward_t(\OptimalArm_t)-\MeanReward_t(\ChosenArm_t)\right)^2
        \cdot
        \II\left(\MeanReward_t(\OptimalArm_t)-\MeanReward_t(\ChosenArm_t)\geq\GapLevel\right)
    }
    \leq
    \frac1{\GainRate*[\GapLevel]}~
    \Expect*{\Uncertainty_t(\ChosenArm_t)}
    +
    \frac{\Deviation*[\GapLevel]}{2t^2},
    \label{eq:def-gain-rate}
\end{align}
for all $t\in[T]$.
\end{defn}
\begin{rem}
The constant $\Deviation*[\GapLevel]$ is meant to account for very unlikely cases where the observations \emph{deviate} from generic cases (this will be formalized by tail bounds). In most cases, $\Deviation*[\GapLevel]$ can be set to 0 or 1.
\end{rem}
We are ready now to state a general result on the regret of any policy for any gap level $\delta$ that relies on uncertainty complexity $\Complexity$, gain rate $\GainRate*[\delta]$, and $\Deviation*[\GapLevel]$.
\begin{thm} Given an uncertainty structure $\{\Uncertainty_t\}_{t\ge1}$, gap level $\delta$, and associated parameter $\GapProb*[\GapLevel]$, the regret of any policy $\Policy$ satisfies
\label{thm:general-regret}
\begin{align}
    \Regret(T,\Policy)
    \leq
    \frac{\Complexity}{\GapLevel\GainRate*[\GapLevel]}+\frac{\Deviation*[\GapLevel]}{\delta}+T\GapLevel\GapProb*[\GapLevel]\,.\label{eq:regret-general-gapped}
\end{align}
\end{thm}
\begin{rem}[Problem-independent bound]\label{rem:prob-ind-bound}
In most examples of this paper we will prove the following stronger variant of \cref{eq:def-gain-rate}:
\begin{align}
    \Expect*{
        \left(\MeanReward_t(\OptimalArm_t)-\MeanReward_t(\ChosenArm_t)\right)^2
    }
    \leq
    \frac1\GainRate~
    \Expect*{\Uncertainty_t(\ChosenArm_t)}
    +
    \frac\Deviation{2t^2}\,.\label{eq:gain-rate-prob-ind}
\end{align}
This inequality, implies that the gain rate is not a function of $\GapLevel$, which means that the regret bound in \cref{eq:regret-general-gapped} holds for any $\GapLevel$. Therefore, one can take the infimum of the right-hand side of \cref{eq:regret-general-gapped} over $\GapLevel$ to get a $\GapLevel$-independent regret bound
\begin{align}
    \Regret(T,\Policy)
    &\leq
    \inf_{\GapLevel>0}\left\{\frac{\Complexity{}}{\GapLevel\GainRate}+\frac{\Deviation}{\delta}+T\GapLevel\GapProb*[\GapLevel]\right\}\nonumber\\
    &\leq
    \inf_{\GapLevel>0}\left\{\frac{\Complexity{}}{\GapLevel\GainRate}+\frac{\Deviation}{\delta}+T\GapLevel\right\}\nonumber\\
    &=
    2\,\sqrt{\left(\frac{\Complexity{}}{\GainRate}+\Deviation\right)\, T}\,.\label{eq:prob-ind-regret-general}
\end{align}
\end{rem}
\begin{rem}\label{rem:gain-rate-similar-to-RV16}
	The above notion of gain rate is similar to the notion of information ratio of \cite{russo2016information}. Specifically, if $\{\Uncertainty_t\}_{t\ge 1}$ is defined as in \cref{exa:bayes-lin-bandit-finite} and $\Deviation=0$, then $1/\GainRate$ becomes the information ratio. We also note that \cite{russo2016information} consider a Bayesian setting while our gain rate is defined for both Bayesian and frequentist settings.
\end{rem}
\begin{proof}[Proof of \cref{thm:general-regret}]
Let $\BigRegret_t$ be the shorthand for the indicator function $\II\left(\MeanReward_t(\OptimalArm_t)-\MeanReward_t(\ChosenArm_t)\geq\GapLevel\right)$. We then have
\begin{align*}
    \Expect*{
        \Reward(\OptimalArm_t)-\Reward(\ChosenArm_t)
    }
    &=
    \Expect*{
        \MeanReward_t(\OptimalArm_t)-\MeanReward_t(\ChosenArm_t)
    }\\
    &=
    \Expect*{
        \left(\MeanReward_t(\OptimalArm_t)-\MeanReward_t(\ChosenArm_t)\right)
        \cdot
        \BigRegret_t
    }+
    \Expect*{
        \left(\MeanReward_t(\OptimalArm_t)-\MeanReward_t(\ChosenArm_t)\right)
        \cdot
        (1-\BigRegret_t)
    }\\
    &\leq
    \frac1\GapLevel
    \Expect*{
        \left(\MeanReward_t(\OptimalArm_t)-\MeanReward_t(\ChosenArm_t)\right)^2
        \cdot
        \BigRegret_t
    }
    +
    \Expect*{
        \left(\MeanReward_t(\OptimalArm_t)-\MeanReward_t(\ChosenArm_t)\right)
        \cdot
        (1-\BigRegret_t)
    }\nonumber\\
    &\leq
    \frac1{\GapLevel\GainRate*[\GapLevel]}~
    \Expect*{\Uncertainty_t(\ChosenArm_t)}
    +
    \frac{\Deviation*[\GapLevel]}{2\GapLevel t^2}
    +
    \Expect*{
        \left(\MeanReward_t(\OptimalArm_t)-\MeanReward_t(\ChosenArm_t)\right)
        \cdot
        (1-\BigRegret_t)
    }\,.
\end{align*}
By summing both sides of the above inequality over $t$, we get
\begin{align}
    \Regret(T,\Policy)
    &=
    \sum_{t=1}^{T}
    \Expect*{
        \Reward(\OptimalArm_t)-\Reward(\ChosenArm_t)
    }\nonumber\\
    &\leq
    \sum_{t=1}^{T}
    \left\{\frac1{\GapLevel\GainRate*[\GapLevel]}~
    \Expect*{\Uncertainty_t(\ChosenArm_t)}
    +
    \frac{\Deviation*[\GapLevel]}{2\GapLevel t^2}
    +
    \Expect*{
        \left(\MeanReward_t(\OptimalArm_t)-\MeanReward_t(\ChosenArm_t)\right)
        \cdot
        (1-\BigRegret_t)
    }\right\}\nonumber\\
    &\leq
    \frac{\Uncertainty(\Policy)}{\GapLevel\GainRate*[\GapLevel]}
    +
       \frac{\Deviation*[\GapLevel]}{\delta}
    +
    \sum_{t=1}^{T}
    \Expect*{
        \left(\MeanReward_t(\OptimalArm_t)-\MeanReward_t(\ChosenArm_t)\right)
        \cdot
        (1-\BigRegret_t)
    }\nonumber\\
    &\leq
    \frac{\Complexity{}}{\GapLevel\GainRate*[\GapLevel]}
    +
    \frac{\Deviation*[\GapLevel]}{\delta}
    +
    \sum_{t=1}^{T}
    \Expect*{
        \left(\MeanReward_t(\OptimalArm_t)-\MeanReward_t(\ChosenArm_t)\right)
        \cdot
        (1-\BigRegret_t)
    }\,. \label{eq:mid-pf-thm-general-regret-to-use-for-poly-log-results}
\end{align}
Finally, note that
\begin{align*}
    \sum_{t=1}^{T}
    \Expect*{
        \left(\MeanReward_t(\OptimalArm_t)-\MeanReward_t(\ChosenArm_t)\right)
        \cdot
        (1-\BigRegret_t)
    }
    &\leq
    \sum_{t=1}^{T}
    \Expect*{
        \GapLevel
        \cdot
        (1-\BigRegret_t)
    }\\
    &=
    \sum_{t=1}^{T}
    \Expect*{
        \GapLevel
        \cdot
        (1-\BigRegret_t)\II(\Gap_t<\GapLevel)
    }\\
    &\leq
    \sum_{t=1}^{T}
    \Expect*{
        \GapLevel
        \cdot
        \II(\Gap_t<\GapLevel)
    }\\
    &\leq
    \sum_{t=1}^{T}
    \GapLevel\GapProb*[\GapLevel]\\
    &=
    T\GapLevel\GapProb*[\GapLevel].
\end{align*}
This completes the proof of \cref{thm:general-regret}.
\end{proof}

\section{ROFUL Algorithm}
\label{sec:optimism}

In this section we generalize the well-known optimism principle that is at the core of the OFUL algorithm of \cite{abbasi2011improved}. Specifically, we introduce the new notion of \emph{optimism in expectation}, which allows us to propose a more general and more flexible version of OFUL, which we call the randomized OFUL (ROFUL) algorithm. We then show how optimism in expectation for a policy leads to a high gain rate and, hence a small regret bound. This allows us to prove a regret bound for ROFUL.
In the next section we will show that, in addition to OFUL, Thompson sampling (TS) is also a special case of ROFUL. We will also see that our regret bound for ROFUL leads to a unified proof of rate optimality for both OFUL and TS.

Before executing the above plan, let us start with a few definitions.
%
\begin{defn}[Confidence bounds]\label{def:conf-bound}
Confidence bounds are real-valued functions $\LowerConf_t(\cdot)$ and $\UpperConf_t(\cdot)$ such that, with probability at least $1-t^{-3}$,
\begin{align*}
    \MeanReward_t(A)\in[\LowerConf_t(\Arm),\UpperConf_t(\Arm)]
    ~~~~~~~~\text{for all $\Arm\in\ArmSet_t$}\,.
\end{align*}
Also, $\ITypicalM_t$ refers to the indicator function for the event that
$\MeanReward_t(A)\in[\LowerConf_t(\Arm),\UpperConf_t(\Arm)]$ holds for all $\Arm\in\ArmSet_t$.
\end{defn}
\begin{defn}[Baseline]\label{def:baselevel}
For confidence bounds  $\LowerConf_t(\cdot)$ and $\UpperConf_t(\cdot)$, the baseline $\BaseLevel_t$ at time $t$ is defined by,
\begin{align*}
\BaseLevel_t
:=
\sup_{\Arm\in\ArmSet_t}\LowerConf_t(\Arm)\,.
\end{align*}
\end{defn}
Next, we state an assumption that allows us to provide results in situations where $\Param$ is unbounded.  In most of the prior literature $\Norm{\Param}_2$ is bounded almost surely, which results in the exclusion of normal priors. The assumption allows us to overcome this constraint.
\begin{asmp}
\label{as:boundedness}
For any constant $\rho\in[ T^{-2},1]$,
let $\BernoulliFamily_{\rho}$ refer to the family of all
Bernoulli random variables $Z$ such that $\Expect{Z}=\rho$. Assume that,
\begin{align*}
    \sup_{Z\in\BernoulliFamily_\rho} \Expect*{\left(\sup_{\Arm\in\ArmSet_t}\MeanReward_t(\Arm)-\inf_{\Arm\in\ArmSet_t}\MeanReward_t(\Arm)\right)^2\cdot Z}
    \leq
    \frac{\Deviation\cdot \rho}{4}\,.
\end{align*}
Note that random variables $Z$ in $\BernoulliFamily_\rho$ can be correlated with $\left(\sup_{\Arm\in\ArmSet_t}\MeanReward_t(\Arm)-\inf_{\Arm\in\ArmSet_t}\MeanReward_t(\Arm)\right)$.
\end{asmp}
The expression $\left(\sup_{\Arm\in\ArmSet_t}\MeanReward_t(\Arm)-\inf_{\Arm\in\ArmSet_t}\MeanReward_t(\Arm)\right)$ in \cref{as:boundedness} is the maximum attainable regret of any policy at time $t$. Applying the Cauchy--Schwarz inequality, we can see that a sufficient condition for \cref{as:boundedness} to hold is that
\[
\sup_{Z\in\BernoulliFamily_\rho} \Expect*{\,\Norm{\Param}_2^2\cdot Z\,}
\leq
\frac{\Deviation\cdot \rho}{4\ArmBound^2}\,.
\]
For example, in the special case where $\Norm{\Param}_2\leq1$ almost surely, the parameter $\Deviation$ can be set to $4\ArmBound^2$.

We are ready now to introduce the ROFUL algorithm.

\paragraph{ROFUL Algorithm.}
ROFUL receives a \emph{worth function} $\Worth_t(\cdot)$ that maps each arm $\Arm\in\ArmSet_t$ and each history instance $\HistoryPlus_{t-1}$ into a real number. The policy then chooses the action with the highest worth. \cref{alg:roful} presents the pseudocode of ROFUL.
\begin{algorithm}[ht]
\caption{Randomized OFUL}
\label{alg:roful}
\begin{algorithmic}[1]
\REQUIRE Worth functions $\left\{\Worth_t(\cdot)\right\}_{t\ge 1}$.
\FOR{$t=1,2,\cdots$}
\STATE Observe $\ArmSet_t$,
\STATE $\ChosenArm_t\gets \Argmax_{\Arm\in\ArmSet_t}\Worth_t(\Arm)$
\ENDFOR
\end{algorithmic}
\end{algorithm}

\paragraph{Regret of ROFUL.} The ROFUL algorithm as formulated in
\cref{alg:roful} may not perform well, unless the worth functions $\Worth_t(\cdot)$ are well behaved. We formally define what ``well behaved'' means by introducing two conditions of \emph{reasonableness} and \emph{optimism}.
Intuitively, an algorithm that explores too much or too little incurs a high regret. Reasonableness  and optimism are mechanisms for controlling these potential flaws, respectively.
To define these notions rigorously, we assume that for each action $\Arm$ we are given upper and lower confidence bounds $\UpperConf_t(\Arm)\geq\LowerConf_t(\Arm)$, where as in \cref{def:conf-bound}, the interval
$[\LowerConf_t(\Arm),\UpperConf_t(\Arm)]$ contains $\MeanReward_t(\Arm)$ with high probability. In \cref{sec:examples}, we provide examples of these confidence bounds for several examples of problems.

We are now ready to define the reasonableness for worth functions.
\begin{defn}[Reasonableness]\label{def:reasonableness}
Given confidence bounds $\LowerConf_t(\cdot)$ and $\UpperConf_t(\cdot)$, a worth function $\Worth_t(\cdot)$ is called \emph{reasonable} if, with probability at least $1-t^{-3}$,
\begin{align*}
    \Worth_t(\Arm)\in[\LowerConf_t(\Arm),\UpperConf_t(\Arm)]
        ~~~~~~~~\text{for all $\Arm\in\ArmSet_t$}\,.
\end{align*}
Moreover, the notation $\ITypicalW_t$ refers to the indicator function for the event that
$\Worth_t(\Arm)\in[\LowerConf_t(\Arm),\UpperConf_t(\Arm)]$  holds for all $\Arm\in\ArmSet_t$.
\end{defn}
As we saw in \cref{def:conf-bound}, the confidence bounds are such that for each arm $\Arm$, the true mean reward $\MeanReward_t(\Arm)$ lies in the confidence interval $[\LowerConf_t(\Arm),\UpperConf_t(\Arm)]$ with high probability. Therefore, reasonableness ensures that the action chosen by ROFUL is close to the best action that ensures that ROFUL does not explore actions unnecessarily.

Next, we  define \emph{optimism in expectation}, which guarantees that ROFUL explores sufficiently.
\begin{defn}[Optimism in expectation]\label{def:optimism-in-exp}
Given confidence bounds $\LowerConf_t(\cdot)$ and $\UpperConf_t(\cdot)$, a worth function $\Worth_t(\cdot)$ is called optimistic with parameter $\OptProb\in(0,1]$, if
\begin{align}
    \Expect*{\Big(\MeanReward_t(\OptimalArm_t)-\BaseLevel_t\Big)^2\cdot\ITypicalM_t}
    \leq
    \frac1\OptProb\Expect*{\Big(\Worth_t(\ChosenArm_t)-\BaseLevel_t\Big)^2\cdot\ITypicalW_t}\,.
    \label{eq:exp-optimism-def}
\end{align}
Note that $\ITypicalM_t$ and $\ITypicalW_t$ are defined as in \cref{def:conf-bound} and \cref{def:reasonableness}, respectively.
\end{defn}
 \cref{fig:worth-function} shows an illustration of the confidence bounds, the baseline, the interval used in optimism, and the worth functions.
 \begin{figure}[th]
 	\centering
 	\includegraphics[scale=0.7]{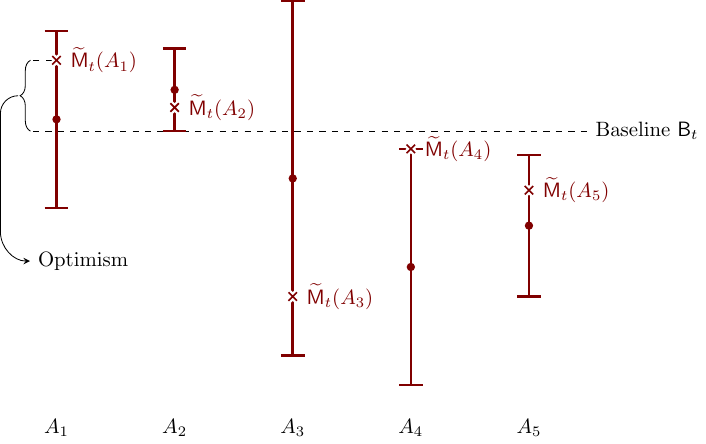}
 	\caption{Illustration of the building blocks of the ROFUL algorithm. Specifically,  confidence bounds, baseline, worth functions, and optimism are shown when $\ArmSet_t=\{\Arm_1,\Arm_2,\ldots,\Arm_5\}$.}
 	\label{fig:worth-function}
 \end{figure}

The above notion requires the ROFUL algorithm to avoid paying the price of pure optimism (as OFUL does). Specifically, OFUL ensures that the inequality
\[
\Big(\MeanReward_t(\OptimalArm_t)-\BaseLevel_t\Big)^2\cdot\ITypicalM_t
\leq
\Big(\Worth_t(\ChosenArm_t)-\BaseLevel_t\Big)^2\cdot\ITypicalW_t
\]
 holds almost surely since for OFUL (as we will see in \cref{sec:examples}) the worth function is $\Worth_t(\ChosenArm_t)=\UpperConf_t(\ChosenArm_t)$. However, the analysis of ROFUL shows that all we need is that the inequality holds in expectation and up to a constant $\OptProb$.
In \cref{sec:sg}, we will leverage the above intuition and introduce our sieved greedy (SG) algorithm that
selects actions more greedily than OFUL while maintaining OFUL's regret guarantees up to a constant. The core idea behind SG is to use data to stay as close as possible to the greedy policy while ensuring that the selected action $\ChosenArm_t$ satisfies the optimism-in-expectation condition.

Next, we show that the gain rate of ROFUL can be controlled by $\OptProb$, when the worth functions are reasonable and optimistic in expectation with parameter $\OptProb$.
\begin{thm}[Gain rate of ROFUL]
\label{thm:roful-gain-rate}
Assume that $\Worth_t(\cdot)$ is reasonable
and optimistic in expectation (with parameter $\OptProb$). Also assume that \cref{as:boundedness} holds with constant $\Deviation$; then we have
\begin{align*}
    \Expect*{
        \Big(\MeanReward_t(\OptimalArm_t)-\MeanReward_t(\ChosenArm_t)\Big)^2
    }
    \leq
    \frac2\OptProb\Expect*{
       \Big(\UpperConf_t(\ChosenArm_t)-\LowerConf_t(\ChosenArm_t)\Big)^2
    }+\frac{\Deviation}{2t^2}\,.
\end{align*}
\end{thm}
Before proving \cref{thm:roful-gain-rate}, we state
our main regret bound for ROFUL which is a corollary of \cref{thm:general-regret}, \cref{rem:prob-ind-bound}, and \cref{thm:roful-gain-rate}.
\begin{cor}[Regret of ROFUL]
\label{cor:roful-regret}
If one defines an uncertainty structure by
\begin{align*}
    \Uncertainty_t(\Arm)=\Big(\UpperConf_t(\Arm)-\LowerConf_t(\Arm)\Big)^2\,,
\end{align*}
\cref{thm:general-regret} implies the following gap-dependent regret bound, for any $\delta$ and $\GapProb*[\GapLevel]$ as in \cref{def:gapped-problem}:
\begin{align*}
    \Regret(T,\PolicyROFUL)
    \leq
    \frac{2\Complexity{}}{\GapLevel\OptProb}+\frac{\Deviation}{\GapLevel}+T\GapLevel\GapProb*[\GapLevel]\,,
\end{align*}
and (by \cref{rem:prob-ind-bound}) the following gap-independent regret bound:
\begin{align*}
    \Regret(T,\PolicyROFUL)
    \leq
   2\sqrt{\left(\frac{2\Complexity{}}{\OptProb}+\Deviation\right)T}\,.
\end{align*}
\end{cor}
\begin{proof}[Proof of \cref{thm:roful-gain-rate}] Define indicator variable $\ITypical_t$ as $\ITypical_t:=\ITypicalW_t\cdot\ITypicalM_t$. Since these are indicator variables, we have $1-\ITypical_t\leq (1-\ITypicalW_t)+(1-\ITypicalM_t)$.
Therefore, by the definition of $\ITypicalW_t$ and of $\ITypicalM_t$, we obtain that,
\[
	\Expect*{1-\ITypical_t}
	\leq	\Expect*{1-\ITypicalW_t}+ \Expect*{1-\ITypicalM_t}\leq 2t^{-3}\leq 2t^{-2}\,.
\]
Using  \cref{as:boundedness} we obtain
\begin{align}
	\Expect*{
		\Big(\MeanReward_t(\OptimalArm_t)-\MeanReward_t(\ChosenArm_t)\Big)^2\cdot(1-\ITypical_t)
	}
	&\leq
	\frac{\Deviation}{2t^2}\,.\label{eq:pf-thm-roful-gain-rate-2}
\end{align}
On the other hand,
\begin{align}
    \Expect*{
        \Big(\MeanReward_t(\OptimalArm_t)-\MeanReward_t(\ChosenArm_t)\Big)^2\cdot\ITypical_t
    }
    &\leq
    \Expect*{
        \Big(\MeanReward_t(\OptimalArm_t)-\LowerConf_t(\ChosenArm)\Big)^2\cdot\ITypical_t
    }\nonumber\\
    &\leq
    2\,\Expect*{
        \Big(\MeanReward_t(\OptimalArm_t)-\BaseLevel_t\Big)^2\cdot\ITypicalM_t+\Big(\BaseLevel_t-\LowerConf_t(\ChosenArm_t)\Big)^2\cdot\ITypical_t
    }\nonumber\\
    &\leq
    \frac2\OptProb\Expect*{
        \Big(\Worth_t(\ChosenArm_t)-\BaseLevel_t\Big)^2\cdot\ITypicalW_t+\Big(\BaseLevel_t-\LowerConf_t(\ChosenArm_t)\Big)^2\cdot\ITypical_t
    }\nonumber\\
    &\leq
    \frac2\OptProb\Expect*{
        \Big(\UpperConf_t(\ChosenArm_t)-\BaseLevel_t\Big)^2\cdot\ITypicalW_t+\Big(\BaseLevel_t-\LowerConf_t(\ChosenArm_t)\Big)^2\cdot\ITypicalW_t
    }\nonumber\\
    &\leq
    \frac2\OptProb\Expect*{
        \Big(\UpperConf_t(\ChosenArm_t)-\LowerConf_t(\ChosenArm_t)\Big)^2\cdot\ITypicalW_t
    }\nonumber\\
    &\leq
    \frac2\OptProb\Expect*{
        \Big(\UpperConf_t(\ChosenArm_t)-\LowerConf_t(\ChosenArm_t)\Big)^2
    }\,.\label{eq:pf-thm-roful-gain-rate-1}
\end{align}
The result now follows by summing both sides of \cref{eq:pf-thm-roful-gain-rate-2} and \cref{eq:pf-thm-roful-gain-rate-1}.
\end{proof} 
\section{Examples of ROFUL and Sieved Greedy}
\label{sec:examples}

The goal of this section is to demonstrate tangible examples of the ROFUL algorithm that may have seemed rather abstract up to this point.
First, in \crefrange{sec:oful-worst-case}{sec:bayesian-ts-finite-arms} we show that OFUL and variations of TS are special cases of ROFUL, which helps us to recover known regret bounds for them via our machinery from \crefrange{sec:uncertainty}{sec:optimism}. Then, in \cref{sec:sg}, motivated by our notion of optimism in expectation and its role in the regret of ROFUL, we introduce a new algorithm (sieved greedy) that enjoys similar theoretical guarantees to those of OFUL, but tends to make more greedy decisions, and hence achieves better empirical performance.

\subsection{Worst-case analysis of OFUL}\label{subsec:OFUL-as-ROFUL}
\label{sec:oful-worst-case}
As our first example, we study the OFUL algorithm of \cite{abbasi2011improved}.
First, building on the notation from \cref{eq:sym-cov-matrix-def}, we define
\begin{align*}
    \Estimator_{t}
    :=
    \SymCovMatrix_t\left(\frac{1}{\sigma^2}\sum_{s=1}^{t}\ChosenArm_s\Reward(\ChosenArm_s)\right)\,.
\end{align*}
Using Theorem 1 of \cite{abbasi2011improved}, we realize that
\begin{align*}
    \Prob*{
        \Norm{\Param-\Estimator_{t}}_{\SymCovMatrix_t^{-1}}
        \geq
        \Radius
    }
    \leq
    \frac{1}{t^3}\,,
\end{align*}
where
\begin{align}
    \Radius
    :=
  \sqrt{d\log\left(1+\frac{T\ArmBound^2}{d\sigma^2}\right)+7\log T}+\frac{1}{\sqrt\lambda}(\ParamBound+\sqrt{7\log T})\,.
    \label{eq:oful-lb-radius}
\end{align}
Therefore,  we can apply the Cauchy--Schwartz inequality and conclude that, for all $A\in\ArmSet_t$,  the following functions satisfy the confidence bounds definition:
\begin{align}
    \LowerConf_t(\Arm)
    :=
    \Inner[\big]{\Estimator_t,\Arm}-\Radius\Norm{\Arm}_{\SymCovMatrix_t}
    ~~~~~~~~~\text{and}~~~~~~~~~
    \UpperConf_t(\Arm)
    :=
    \Inner[\big]{\Estimator_t,\Arm}+\Radius\Norm{\Arm}_{\SymCovMatrix_t}\,.\label{eq:CI-via-rho-sigma}
\end{align}
Moreover, OFUL can be written as an instance of ROFUL as follows:
\begin{align*}
    \Worth_t(\Arm):=\UpperConf_t(\Arm)\,.
\end{align*}
Reasonableness follows from the definition of this worth function and $\ITypicalW_t$ will be always equal to $1$. For optimism, note that, whenever $\ITypicalM_t=1$, we have
\begin{align*}
    \Worth_t(\ChosenArm_t)
    &=
    \UpperConf_t(\ChosenArm_t)\\
    &\geq
    \UpperConf_t(\OptimalArm_t)\\
    &\geq
    \MeanReward_t(\OptimalArm_t)\\
    &\geq
    \BaseLevel_t\,.
\end{align*}
We thus get
\begin{align*}
    \Expect*{\Big(\MeanReward_t(\OptimalArm_t)-\BaseLevel_t\Big)^2\cdot\ITypicalM_t}
    \leq
    \Expect*{\Big(\Worth_t(\ChosenArm_t)-\BaseLevel_t\Big)^2\cdot\ITypicalW_t}\,.
\end{align*}
This in turn implies that the optimism in expectation holds with $\OptProb=1$. Using \cref{cor:roful-regret} together with \cref{eq:unstr-lin-bandit-complexity-bound} leads to the following gap-dependent bound:
\begin{align*}
    \Regret(T,\PolicyOFUL)
    \leq
    \frac{16\sigma^2\Radius^2d}{\GapLevel}\log\left(1+\frac{T\ArmBound^2}{d\sigma^2}\right)+\frac{\Deviation}{\GapLevel}+T\GapLevel\GapProb*[\GapLevel]\,,
\end{align*}
and the following gap-independent bound:
\begin{align*}
    \Regret(T,\PolicyOFUL)
    \leq
    2\sqrt{\left[16 \sigma^2\Radius^2d\log\left(1+\frac{T\ArmBound^2}{d\sigma^2}\right)+\Deviation \right]T }\,.
\end{align*}
Note that, if we ignore logarithmic factors, this bound is $\Order{d\sqrt{T}}$ since $\Radius$ is $\Order{\sqrt{d}}$ and $\Deviation$ is constant.

\subsection{Bayesian analysis of TS}\label{subsec:bayesian-TS-as-ROFUL}
We obtain a Bayesian regret upper bound for TS similar to the one proved by \cite{russo2014learning}. Let $\Estimator_{t}$, $\Radius$, $\LowerConf_t$, and $\UpperConf_t$ be given as in \cref{sec:oful-worst-case}. Unlike in the previous section where $\Param$ was fixed, here we assume that $\Param$ is also drawn from a prior distribution. Define the worth function by
\begin{align*}
    \Worth_t(\Arm)
    :=
    \Inner[\big]{\TsSample_t,\Arm}\,,
\end{align*}
where $\TsSample_t$ is a sample drawn from the posterior distribution of $\Param$ at time $t$ that is used in TS. Therefore, $\TsSample_t$ and $\Param$ are exchangeable, given $\HistoryPlus_{t-1}$; this, together with the definition of TS, gives
\begin{align*}
    \Expect*{\Big(\MeanReward_t(\OptimalArm_t)-\BaseLevel_t\Big)^2\cdot\ITypicalM_t\Given\HistoryPlus_{t-1}}
    =
    \Expect*{\Big(\Worth_t(\ChosenArm_t)-\BaseLevel_t\Big)^2\cdot\ITypicalW_t\Given\HistoryPlus_{t-1}},
\end{align*}
almost surely. This implies optimism in expectation with $\OptProb=1$. For reasonableness, we can leverage the same argument and obtain
\begin{align*}
    \Prob*{\forall\Arm\in\ArmSet_t:\Worth_t(\Arm)\in[\LowerConf_t(\Arm),\UpperConf_t(\Arm)]
    }
    &=
    \Expect*{
        \Prob[\bigg]{\forall\Arm\in\ArmSet_t:\Worth_t(\Arm)\in[\LowerConf_t(\Arm),\UpperConf_t(\Arm)]
        \Given
        \HistoryPlus_{t-1}
    }}\\
    &=
    \Expect*{
        \Prob[\bigg]{\forall\Arm\in\ArmSet_t:\MeanReward_t(\Arm)\in[\LowerConf_t(\Arm),\UpperConf_t(\Arm)]
        \Given
        \HistoryPlus_{t-1}
    }}\\
    &=
    \Prob[\Big]{
        \forall\Arm\in\ArmSet_t:\MeanReward_t(\Arm)\in[\LowerConf_t(\Arm),\UpperConf_t(\Arm)]
    }\\
    &\leq
    \frac1{t^3}\,.
\end{align*}
Hence, we obtain the same gap-dependent bound of
\begin{align*}
	\Regret(T,\PolicyTS)
	\leq
	\frac{16\sigma^2\Radius^2d}{\GapLevel}\log\left(1+\frac{T\ArmBound^2}{d\sigma^2}\right)+\frac{\Deviation}{\GapLevel}+T\GapLevel\GapProb*[\GapLevel]\,,
\end{align*}
and the same gap-independent bound of
\begin{align*}
	\Regret(T,\PolicyTS)
	\leq
	2\sqrt{\left[16 \sigma^2\Radius^2d\log\left(1+\frac{T\ArmBound^2}{d\sigma^2}\right)+\Deviation \right]T }\,.
\end{align*}

\subsection{Worst-case analysis of TS}\label{subsec:frequentist-TS-as-ROFUL}

In this section, we study the worst-case (frequentist) regret of TS with inflated posterior variance. We recover the same bounds as the ones by \cite{agrawal2013thompson} and \cite{abeille2017linear}.  \cref{alg:inflated-ts} shows the pseudocode for this instance of TS. We also make the additional assumption that  $|\ArmSet_t|\le n$ for all $t$.
\begin{algorithm}[t]
\caption{Linear Thompson sampling with inflated posterior}
\label{alg:inflated-ts}
\begin{algorithmic}[1]
\REQUIRE Inflation rate $\Inflation$.
\STATE Initialize $\SymCovMatrix_1\gets\frac1\lambda\Eye_d$ and $\Estimator_1\gets0$
\FOR{$t=1,2,\cdots$}
\STATE Observe $\ArmSet_t$
\STATE Sample $\TsSample_t\sim\Normal{\Estimator_t,\:\Inflation^2\SymCovMatrix_t}$
\STATE $\ChosenArm_t\gets \Argmax_{\Arm\in\ArmSet_t}\Inner[\big]{\Arm,\Estimator_t}$
\STATE Observe reward $\Reward(\ChosenArm_t)$
\STATE $\SymCovMatrix_{t+1}^{-1}\gets\SymCovMatrix_t^{-1}+\frac1{\sigma^2}\ChosenArm_t\ChosenArm_t^\top$
\STATE $\Estimator_{t+1}\gets \SymCovMatrix_{t+1}\left(\SymCovMatrix_t^{-1}\Estimator_t+\frac1{\sigma^2}\ChosenArm_t\Reward(\ChosenArm_t) \right)$
\ENDFOR
\end{algorithmic}
\end{algorithm}

Due to the inflated variance, we need to redefine $\LowerConf_t(\cdot)$ and $\UpperConf_t(\cdot)$. Specifically, let
\begin{align*}
    \Radius':=\max\left\{\Radius, \Inflation\sqrt{\min\left\{2\Dim+12\log(T),6\log(2nT)\right\}} \right\}\,,
\end{align*}
and define
\begin{align*}
    \LowerConf_t(\Arm)
    :=
    \Inner[\big]{\Estimator_t,\Arm}-\Radius'\Norm{\Arm}_{\SymCovMatrix_t}
    ~~~~~~~~~\text{and}~~~~~~~~~
    \UpperConf_t(\Arm)
    :=
    \Inner[\big]{\Estimator_t,\Arm}+\Radius'\Norm{\Arm}_{\SymCovMatrix_t}\,.
\end{align*}
As $\Radius'\geq\Radius$, we infer that $\LowerConf_t(\cdot)$ and $\UpperConf_t(\cdot)$ satisfy the confidence bounds condition (\cref{def:conf-bound}). We note that this definition replaces the $\Radius^2$ term in $\Complexity$ with $\Radius'^2$.
Next, we prove that the worth function given by
\begin{align}
    \Worth_t(\Arm)=\Inner[\big]{\TsSample_t,\Arm}
    \label{eq:ts-worth}
\end{align}
is reasonable. This is achieved by the following lemma, which is proved in \cref{sec:auxi}.
\begin{lem}
\label{lem:optimism-for-ts}
For all $t\in[T]$, we have
\begin{align*}
    \Prob*{\forall\Arm\in\ArmSet_t:\Worth_t(\Arm)\in[\LowerConf_t(\Arm),\UpperConf_t(\Arm)]}
    &\geq
    1-\frac1{2t^3}.
\end{align*}
\end{lem}
In order to derive our regret bound, we also need to verify the optimism in expectation assumption. Whenever $\Inner[\big]{\Estimator_t-\Param,\OptimalArm_t}\geq    -\Radius\Norm[\big]{\OptimalArm_t}_{\SymCovMatrix_t}$, we have
\begin{align*}
    \Worth_t(\ChosenArm_t)-\MeanReward_t(\OptimalArm_t)
    &\geq
    \Worth_t(\OptimalArm_t)-\MeanReward_t(\OptimalArm_t)\\
    &=
    \Inner[\big]{\TsSample_t-\Param,\OptimalArm_t}\\
    &=
    \Inner[\big]{\TsSample_t-\Estimator_t,\OptimalArm_t}
    +
    \Inner[\big]{\Estimator_t-\Param,\OptimalArm_t}\\
    &\geq
    \Inner[\big]{\TsSample_t-\Estimator_t,\OptimalArm_t}
    -
    \Radius\Norm[\big]{\OptimalArm_t}_{\SymCovMatrix_t}\,.
\end{align*}
Since $\Inner[\big]{\TsSample_t-\Estimator_t,\OptimalArm_t}$ is distributed as $\Normal*{0,\Inflation^2\Norm[\big]{\OptimalArm_t}_{\SymCovMatrix_t}^2}$, we can deduce that
\begin{align*}
    \Prob*{
        \Worth_t(\ChosenArm_t)
        \geq
        \MeanReward_t(\OptimalArm_t)
    }
    &\geq
    \Phi\left(-\frac{\Radius}{\Inflation}\right)
    \Prob*{\Inner[\big]{\Estimator_t-\Param,\OptimalArm_t}\geq    -\Radius\Norm[\big]{\OptimalArm_t}_{\SymCovMatrix_t}}\\
    &\geq
    \frac12\Phi\left(-\frac{\Radius}{\Inflation}\right)\,.
\end{align*}
Finally, since $\ITypicalM_t=1$ and $\ITypicalW_t=1$ imply $\MeanReward_t(\OptimalArm_t)\geq\BaseLevel_t$ and $\Worth_t(\OptimalArm_t)\geq\BaseLevel_t$, respectively, it follows that
\begin{align*}
    \Prob*{
        (\Worth_t(\ChosenArm_t)-\BaseLevel_t)^2\cdot\ITypicalW_t
        \geq
        (\MeanReward_t(\OptimalArm_t)-\BaseLevel_t)^2\cdot\ITypicalM_t
        \Given
        \HistoryPlus_{t-1}
    }
    &\geq
    \frac12\Phi\left(-\frac{\Radius}{\Inflation}\right)-\frac1{T^3}\,.
\end{align*}
Therefore, for sufficiently large $T$, we have
\begin{align*}
    \Prob*{
        (\Worth_t(\ChosenArm_t)-\BaseLevel_t)^2\cdot\ITypicalW_t
        \geq
        (\MeanReward_t(\OptimalArm_t)-\BaseLevel_t)^2\cdot\ITypicalM_t
        \Given
        \HistoryPlus_{t-1}
    }
    &\geq
    \frac14\Phi\left(-\frac{\Radius}{\Inflation}\right)\,.
\end{align*}
Noting that $(\MeanReward_t(\OptimalArm_t)-\BaseLevel_t)^2\cdot\ITypicalM_t$ is deterministic conditional on $\HistoryPlus_{t-1}$, we have
\begin{align*}
    \Expect*{
        (\Worth_t(\ChosenArm_t)-\BaseLevel_t)^2\cdot\ITypicalW_t
    }
    &\geq
    \frac14\Phi\left(-\frac{\Radius}{\Inflation}\right)\cdot
    \Expect*{
        (\MeanReward_t(\OptimalArm_t)-\BaseLevel_t)^2\cdot\ITypicalM_t
    }\,.
\end{align*}
Therefore, optimism in expectation holds with $\OptProb=\Phi(-{\Radius}/{\Inflation})/2$.

Thus, similar to \cref{subsec:bayesian-TS-as-ROFUL}, \cref{cor:roful-regret} gives a gap-dependent bound of
\begin{align*}
	\Regret(T,\PolicyTS)
	\leq
	\frac{32\sigma^2\Radius'^2d}{\GapLevel\Phi\left(-\frac{\Radius}{\Inflation}\right)}\log\left(1+\frac{T\ArmBound^2}{d\sigma^2}\right)+\frac{\Deviation}{\GapLevel}+T\GapLevel\GapProb*[\GapLevel]\,,
\end{align*}
and a similar gap-independent bound of
\begin{align*}
	\Regret(T,\PolicyTS)
	\leq
	2\sqrt{\left[\frac{32 \sigma^2\Radius'^2d}{\Phi\left(-\frac{\Radius}{\Inflation}\right)}\log\left(1+\frac{T\ArmBound^2}{d\sigma^2}\right)+\Deviation \right]T }\,.
\end{align*}

\subsection{Bayesian analysis of TS with finitely many arms}\label{sec:bayesian-ts-finite-arms}

Following the same technique as in the previous example, we can prove a sharper regret bound for TS in the normal prior and normal noise setting. The main idea is to use smaller confidence bounds. More precisely, write
\begin{align*}
    \Radius'':=\sqrt{6\log(2nT)}
\end{align*}
and then define
\begin{align*}
    \LowerConf_t(\Arm)
    :=
    \Inner[\big]{\Estimator_t,\Arm}-\Radius''\Norm{\Arm}_{\SymCovMatrix_t}
    ~~~~~~~~~\text{and}~~~~~~~~~
    \UpperConf_t(\Arm)
    :=
    \Inner[\big]{\Estimator_t,\Arm}+\Radius''\Norm{\Arm}_{\SymCovMatrix_t}\,.
\end{align*}
Using the same techniques as in the proof of reasonableness in the previous example, we can show that these functions satisfy the confidence bounds condition and that the worth function defined by \cref{eq:ts-worth} is reasonable with respect to these confidence bounds. This yields an $\Order[\big]{\sqrt {dT\log(T)\log(nT)}}$ regret bound that is sharper than the well-known $\Order[\big]{d\log(T)\sqrt {T}}$ regret bound. The only comparable result that we are aware of is the $\Order[\big]{\sqrt{dT\Entropy(\OptimalArm)}}$ bound provided by \cite{russo2016information}. Although their bound does not require normality and is sharper than ours, it does not allow changing action sets as does our bound.


\subsection{Toward a better use of data: sieved greedy (SG)}\label{sec:sg}

\begin{algorithm}[ht]
\caption{Sieved-Greedy (SG)}
\label{alg:sg}
\begin{algorithmic}[1]
\REQUIRE $\Radius$, $\lambda$, $\sigma$, $\SieveRate$.
\STATE Initialize $\SymCovMatrix_1\gets\frac1\lambda\Eye_d$ and $\Estimator_1\gets0$
\FOR{$t=1,2,\cdots$}
\STATE Observe $\ArmSet_t$
\STATE Define $\LowerConf_t(\Arm)
    :=
    \Inner[\big]{\Estimator_t,\Arm}-\Radius\Norm{\Arm}_{\SymCovMatrix_t}$
    and
    $\UpperConf_t(\Arm)
    :=
    \Inner[\big]{\Estimator_t,\Arm}+\Radius\Norm{\Arm}_{\SymCovMatrix_t}.$
\STATE Construct $\ArmSet'_t:=\left\{\Arm\in\ArmSet_t:\UpperConf_t(\Arm)
    \ge
    \SieveRate
    \Big(\max_{\Arm'\in\ArmSet_t}
    \UpperConf_t(\Arm')-\BaseLevel_t\Big)
    +
    \BaseLevel_t\right\}$
\STATE $\ChosenArm_t\gets \Argmax_{\Arm\in\ArmSet_t'}\Inner[\big]{\Estimator_t,\Arm}$
\STATE Observe reward $\Reward(\ChosenArm_t)$
\STATE $\SymCovMatrix_{t+1}^{-1}\gets\SymCovMatrix_t^{-1}+\frac1{\sigma^2}\ChosenArm_t\ChosenArm_t^\top$
\STATE $\Estimator_{t+1}\gets \SymCovMatrix_{t+1}\left(\SymCovMatrix_t^{-1}\Estimator_t+\frac1{\sigma^2}\ChosenArm_t\Reward(\ChosenArm_t)\right)$
\ENDFOR
\end{algorithmic}
\end{algorithm}

In this section, we present a novel algorithm that enjoys the same regret bound as the one we proved for OFUL. This new policy, nonetheless, tends to make more greedy decisions. As we will see in \cref{sec:simulation}, this algorithm achieves a similar cumulative regret to that of the best policy in each scenario.
\begin{figure}[th]
    \centering
    \includegraphics[scale=0.7]{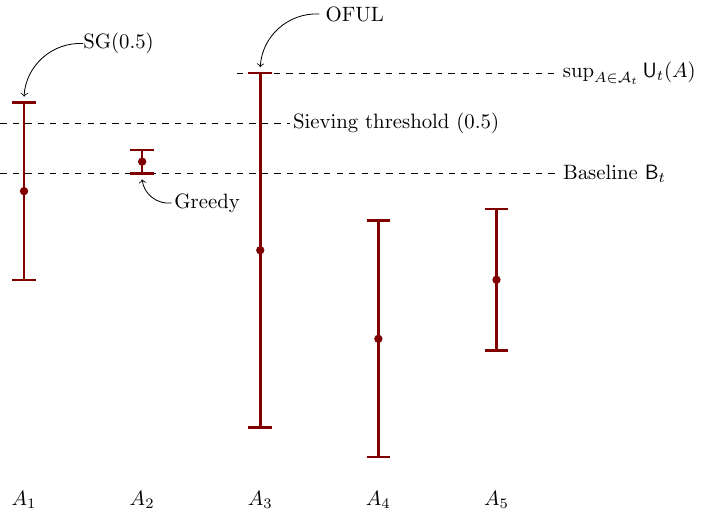}
    \caption{Illustration of how SG works compared to OFUL and greedy.}
    \label{fig:sg}
\end{figure}

This algorithm receives a \emph{sieving-rate} parameter $\SieveRate$ as input. Then, at time $t$, this algorithm first discards all the actions that lack sufficient uncertainty, i.e., that satisfy
\begin{align*}
    \UpperConf_t(\Arm)
    <
    \SieveRate
    \left(\max_{\Arm'\in\ArmSet_t}
    \UpperConf_t(\Arm')-\BaseLevel_t\right)
    +
    \BaseLevel_t\,.
\end{align*}
Denoting the set of remaining (sieved) actions by $\ArmSet'_t$, we note that the algorithm makes a greedy decision over $\ArmSet'_t$, i.e.,
\begin{align*}
    \ChosenArm_t\in\Argmax_{\Arm\in\ArmSet'_t}\Inner[\big]{\Estimator_t,\Arm}\,.
\end{align*}
%
Therefore, we call the algorithm sieved greedy (SG). When $\SieveRate=1$, this algorithm is identical to OFUL and $\SieveRate=0$ leads to the greedy algorithm. \cref{alg:sg} shows the pseudocode for SG and \cref{fig:sg} is an illustration of how SG works.

We show that this algorithm is also an instance of ROFUL. To do so, we introduce the following worth function:
\begin{align}\label{eq:sg-as-roful}
    \Worth_t(\Arm)
    :=
    \begin{cases}
    \UpperConf_t(\Arm)~~~~~~~ & \text{if $\Arm=\ChosenArm_t$,}\\
    \LowerConf_t(\Arm)        & \text{otherwise.}
    \end{cases}
\end{align}
We need to show that the ROFUL algorithm with this worth function chooses the same action as SG. Notice that
\begin{align*}
    \sup_{\Arm\in\ArmSet_t\setminus\{\ChosenArm_t\}}\LowerConf_t(\Arm)
    &\leq
    \BaseLevel_t
    \leq
    \UpperConf_t(\ChosenArm_t).
\end{align*}
Next, the reasonableness of the worth function is evident from the definition of $\Worth_t(\cdot)$. For the optimism, note that
\begin{align*}
    \Expect*{\big(\MeanReward_t(\OptimalArm_t)-\BaseLevel_t\big)^2\cdot\ITypicalM_t}
    &\leq
    \Expect*{\big(\sup_{\Arm\in\ArmSet_t}\UpperConf_t(\Arm)-\BaseLevel_t\big)^2}\\
    &\leq
    \frac{1}{\SieveRate^2}
    \Expect*{\big(\Worth_t(\ChosenArm_t)-\BaseLevel_t\big)^2}.
\end{align*}
Therefore, optimism in expectation holds with $\OptProb=\SieveRate^2$, which means that \cref{cor:roful-regret} gives the gap-dependent bound of
\begin{align*}
	\Regret(T,\PolicySG)
	\leq
	\frac{16\sigma^2\Radius^2d}{\GapLevel\SieveRate^2}\log\left(1+\frac{T\ArmBound^2}{d\sigma^2}\right)+\frac{\Deviation}{\GapLevel}+T\GapLevel\GapProb*[\GapLevel]\,,
\end{align*}
and the gap-independent bound of
\begin{align*}
	\Regret(T,\PolicySG)
	\leq
	2\sqrt{\left[\frac{16 \sigma^2\Radius^2d}{\SieveRate^2}\log\left(1+\frac{T\ArmBound^2}{d\sigma^2}\right)+\Deviation \right]T }\,.
\end{align*} 
\begin{rem}[Sieved version of general ROFUL]
Here we showed that SG is an instance of ROFUL. However, as shown in \cref{eq:sg-as-roful},  this reduction is general. Specifically,  $\ChosenArm_t$ need not to be selected greedily. In fact, any action that is selected from the set of sieved actions can be replaced with $\ChosenArm_t$, and the above regret analysis of SG stays valid. This means one can apply the sieving idea to any instance of ROFUL, including TS and OFUL. We expect SG to outperform such ``sieved TS" or ``sieved OFUL", at least empirically, because it makes more greedy decisions. But, there could be other circumstances, under which, sieved TS or sieved OFUL may be more preferred. For example, a decision-maker may prefer TS as it is a randomized policy, and in such a scenario, she can use sieved TS, with the same theoretical guarantees as SG, while maintaining a TS-based policy.
\end{rem}

\section{Numerical Simulations}
\label{sec:simulation}
In this section, we compare the  performance of OFUL, TS, greedy, and SG (with sieving rates $0,2$, $0.5$, and $0.8$) in two scenarios. In each scenario, the unknown parameter vector $\Param$ is first sampled from $\Normal{0,\Eye_d}$, where $d=120$. Then, in round $t$, a set of $n=10$ actions is generated. More precisely:

\paragraph{Scenario I.} A random vector $V_t$ is picked uniformly at random on the sphere of radius 5 in $\IR^{12}$. Then, the action $\Arm_{t,i}\in\IR^{120}$ for $i=1,2,\ldots,10$ is constructed by copying $V_t$ into the $i$-th block of size 12. Although $d=120$,  this scenario is equivalent to a $10$-armed $12$-dimensional contextual bandit problem with a shared feature vector $V_t$, embedded in the linear bandit framework, as explained by \cite{abbasi-yadkori2012online}.

\paragraph{Scenario II.} Motivated by the more general linear bandit problem,  each action is chosen uniformly at random on the sphere of radius 5 in $\IR^{120}$.

Each policy $\Policy$ chooses an action $\ChosenArm_t^\Policy\in\Arm_t$ and receives the reward $\Reward(\ChosenArm_t^\Policy)=\Inner{\Param,\ChosenArm_t^\Policy}+\eps_t$, where $\eps_t$ is a sequence of i.i.d. standard normal random variables. We run each experiment for $T=10,000$ rounds and repeat this procedure 50 times. The average regret of each policy (and error bars of width $2\times\text{SD}$ in each direction) is shown in \cref{fig:regret}. As is clear from the plots, in Scenario I, TS is the best policy and SG has a very similar performance, while greedy performs very poorly. But in Scenario II, greedy and SG achieve a substantially better performance compared to OFUL and TS. We also see that the performance of SG is generally less dependent on the sieving rate $\alpha$. Specifically, in Scenario II, all versions have the same performance as greedy. In Scenario I, while all versions outperform OFUL and greedy, $\alpha=0.5$ slightly outperforms the other two variants and nearly ties with TS.

These results underscore that SG inherits beneficial properties of both greedy and OFUL. It performs similar to greedy when greedy works well, but does not prematurely drop potentially optimal arms, which causes greedy to perform very poorly sometimes.

\begin{figure}[h]
	\begin{center}
		\begin{subfigure}{0.49\textwidth}
			\includegraphics[width=\textwidth]{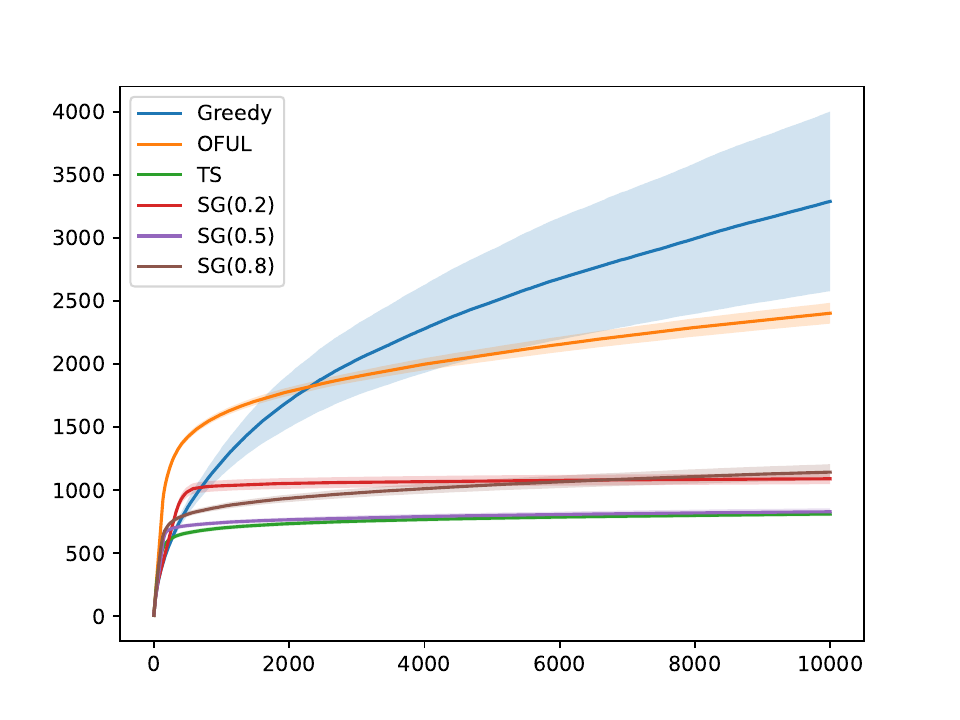}
			\caption{Scenario I}
		\end{subfigure}
		\hfill
		\begin{subfigure}{0.49\textwidth}
			\includegraphics[width=\textwidth]{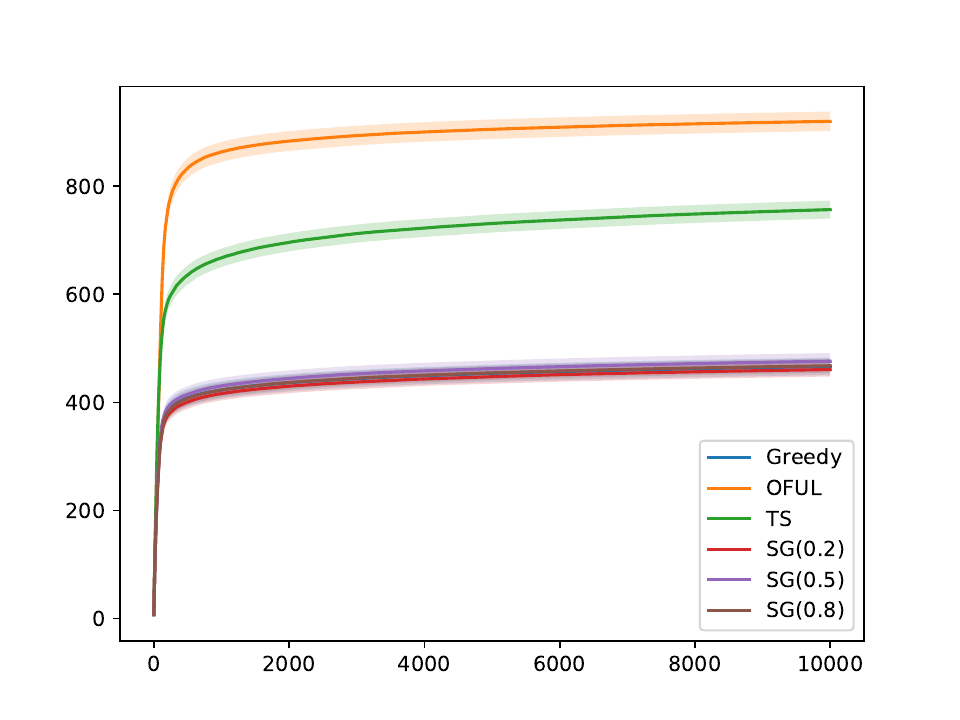}
			\caption{Scenario II}
		\end{subfigure}
\caption{Comparison of the cumulative regret incurred by SG with sieving rates 0.2, 0.5, and 0.8 to the cumulative regret of greedy, TS, and OFUL.}\label{fig:regret}
\end{center}
\end{figure}
\section{Improved Bounds for Two Important Subproblems}
\label{sec:margin}

We can strengthen our regret bounds for ROFUL for two important special cases of the stochastic linear bandit problem. Specifically, in \cref{sec:grouped-problems}, motivated by the $k$-armed contextual bandit problem, we show how our proof technique allows improving all regret bounds of \cref{sec:examples} by a factor $\sqrt{k}$.
Then, in \cref{sec:poly-log-results}, making similar (generalized gap and margin) assumptions to those by \cite{goldenshluger2013linear} and \cite{bastani2020online}, we obtain polylogarithmic regret bounds for ROFUL and the obtain the first such results for OFUL and TS.

\subsection{Grouped linear bandit}\label{sec:grouped-problems}

Here we focus on improving our previous regret bounds for a family of subproblems. Although these improvements are mainly motivated by the special case of the $k$-armed contextual bandit, we formulate a slightly more general case of the stochastic linear bandit which we will refer to as the \emph{grouped linear bandit.}
\begin{defn}[Grouped linear bandit]\label{def:glb}
Let $k$ and $d$ be two integers and let $(\GroupSpace_j)_{j=1}^k$ be a sequence of $d$-dimensional subspaces of $\IR^{kd}$ such that each vector $v\in\IR^{kd}$ can be uniquely decomposed to $v=\sum_{j=1}^kz_j$, where $z_j\in\GroupSpace_j$, i.e.,
$\IR^{kd}=\GroupSpace_1\Orthsum\GroupSpace_2\Orthsum\cdots\Orthsum\GroupSpace_k$. Then, a \emph{grouped linear bandit} (GLB) problem simply refers to a linear bandit problem in which $\ArmDomain\subseteq\bigcup_{j=1}^k\GroupSpace_j$.
\end{defn}
As mentioned above, the GLB formulation is meant to capture the specific structure in the contextual setting. In fact, a $k$-armed $d$-dimensional contextual bandit problem can be modeled as a $kd$-dimensional linear bandit one, as discussed by \cite{abbasi-yadkori2012online}. However, the GLB problem also includes the original linear bandit problem if we assume that $k=1$.
Notice that in the GLB problem we also have $\Param\in\IR^{kd}$, which in turn implies that the number of parameters is $kd$ (rather than $d$). Therefore, our previous problem-independent regret bounds from \cref{sec:examples} would be $\Order{kd\sqrt{T}}$ for this $kd$-dimensional problem. As we will see shortly, this bound can be tightened by a factor of $\sqrt{k}$ to $\Order{d\log(T)\sqrt{kT}}$.

The key observation is that the radius of the confidence set can be shrunk to
\begin{align}
    \GroupedRadius
    :=
    \sqrt{d\log\left(1+\frac{T\ArmBound^2}{d\sigma^2}\right)+7\log (kT)}+\frac{1}{\sqrt\lambda}\left(\ParamBound+\sqrt{7\log (kT)}\right)\,.
    \label{eq:oful-glb-radius}
\end{align}
Note that $\Radius$ as defined in \cref{eq:oful-lb-radius} for this problem is given by
\begin{align*}
    \Radius
    :=
   \sqrt{kd\log\left(1+\frac{T\ArmBound^2}{kd\sigma^2}\right)+7\log T}+\frac{1}{\sqrt\lambda}(\ParamBound+\sqrt{7\log T})\,.\,,
\end{align*}
which is worse than $\Radius$  as defined in \cref{eq:oful-glb-radius} by an asymptotic factor of $\sqrt{k}$ as $T$ grows large. Specifically, we can show that the following functions satisfy the confidence bounds definition:
\begin{align*}
	\LowerConf_t(\Arm)
	:=
	\Inner[\big]{\Estimator_t,\Arm}-\GroupedRadius\Norm{\Arm}_{\SymCovMatrix_t}
	~~~~~~~~~\text{and}~~~~~~~~~
	\UpperConf_t(\Arm)
	:=
	\Inner[\big]{\Estimator_t,\Arm}+\GroupedRadius\Norm{\Arm}_{\SymCovMatrix_t}\,.
\end{align*}
This can be shown by noting that when $\Arm\in \GroupSpace_j$, then $\Inner{\Estimator_t-\Param,\Arm}$ can be bounded by applying Theorem 1 of \cite{abbasi2011improved} in a $d$-dimensional rather than $kd$-dimensional setting. Combining this with the union bound, we obtain
\begin{align*}
        \Prob[\Big]{
    	\exists \Arm\in\ArmSet_t:
    	\MeanReward_t(A)\notin[\LowerConf_t(\Arm),\UpperConf_t(\Arm)]
    	}&\le
    \sum_{j=1}^k\Prob[\Big]{
    	\text{For some }A\in\GroupSpace_j,~
    	  \MeanReward_t(A)\notin[\LowerConf_t(\Arm),\UpperConf_t(\Arm)]
    } \\
   & \le
   \sum_{j=1}^k\Prob*{\text{For some }A\in\GroupSpace_j,~
       \frac{|\Inner{\Estimator_t-\Param,\Arm}|}{\Norm{\Arm}_{\SymCovMatrix_{t}}}
        \geq
        \GroupedRadius
    }\\
   & \leq
   \frac{1}{T^3}\,.
\end{align*}
By the same argument as in \cref{sec:oful-worst-case}, we get
\begin{align*}
    \Regret(T,\PolicyOFUL)
    \leq
    \frac{16\sigma^2\GroupedRadius^2kd}{\GapLevel}\log\left(1+\frac{T\ArmBound^2}{kd\sigma^2}\right)+\frac{\Deviation}{\GapLevel}+T\GapLevel\GapProb*[\GapLevel]\,,
\end{align*}
and by tuning $\delta$ as before we get
\begin{align*}
    \Regret(T,\PolicyOFUL)
    \leq
    2\sqrt{\left[16 \sigma^2\GroupedRadius^2kd\log\left(1+\frac{T\ArmBound^2}{kd\sigma^2}\right)+\Deviation \right]T }\,.
\end{align*}
Moreover, in the Bayesian setting, one can use these confidence bounds to prove a similar Bayesian regret bound for TS with the proper update rule. In the frequentist setting, on the other hand, this idea can be used to show that
\begin{align*}
    \Regret(T,\PolicyTS)
    \leq
    \frac{32\sigma^2\GroupedRadius'^2kd}{\GapLevel\Phi\left(-\frac{\GroupedRadius}{\Inflation}\right)}\log\left(1+\frac{T\ArmBound^2}{kd\sigma^2}\right)+\frac{\Deviation}{\GapLevel}+T\GapLevel\GapProb*[\GapLevel]\,,
\end{align*}
and by tuning $\delta$ as before,
\begin{align*}
    \Regret(T,\PolicyTS)
    \leq
    2\sqrt{\left[\frac{32 \sigma^2\GroupedRadius'^2kd}{\Phi\left(-\frac{\GroupedRadius}{\Inflation}\right)}\log\left(1+\frac{T\ArmBound^2}{kd\sigma^2}\right)+\Deviation \right]T }
\end{align*}
where
\begin{align*}
    \GroupedRadius':=
    \max\left\{\GroupedRadius, \Inflation\sqrt{\min\left\{2k\Dim+12\log(T),6\log(2nT)\right\}} \right\}\,.
\end{align*}
This shows that the posterior variance inflation can be reduced by a factor of $\Order{\sqrt k}$ in TS as $T$ grows large.

\subsection{Polylogarithmic regret bounds}\label{sec:poly-log-results}

In this subsection we provide regret bounds for ROFUL when confidence sets are defined such that they grow with $T$ polylogarithmically, under additional assumptions.
 Our assumptions are similar to those made by \cite{goldenshluger2013linear} and \cite{bastani2020online}.
Throughout this section we consider special classes of ROFUL where confidence intervals (\cref{def:conf-bound}) are defined as in \cref{eq:CI-via-rho-sigma} with varying definitions of $\rho$. This special class includes all examples of \cref{sec:examples} such as OFUL, TS, and SG algorithms.

To state the first assumption, recall the gap parameters $\Gap_t$ and $\GapLevel$ from \S \ref{sec:setting}.
\begin{asmp}[Margin condition]
\label{as:margin-condition}
There exists constants $c_0,t_0>0$ such that
\begin{align}
    \Prob*{\Gap_t\leq z}\leq c_0z
    \label{eq:margin-condition}
\end{align}
for all $0\leq z\leq \GapLevel$ and $t\in[T]$ with $t\geq t_0$.
\end{asmp}
Before stating the next condition, we need to define a notion of \emph{near-optimal space} for the family of GLB problems.
\begin{defn}[Near-optimal space]
	Consider a GLB problem as defined in \cref{def:glb}. Let $l\leq k$ be the smallest number such that there exists $\Iset\subseteq[k]$ with $|\Iset|=l$ and
	\begin{align*}
				\Prob*{
					\ArmSet_t^\GapLevel\subseteq \Orthsum_{j\in\Iset}\GroupSpace_j
				}=1
		~~~~~~\text{for all $t\in[T]$}\,.
	\end{align*}
We define near-optimal space $\NearOptimal$ as $\NearOptimal:=\Orthsum_{j\in\Iset}\GroupSpace_j$ and, with a slight abuse of notation, we also treat $\NearOptimal$ as the projection of $\IR^{kd}$ onto the subspace $\NearOptimal$.
\end{defn}
\begin{rem}
	The main purpose of this notion is to handle suboptimal arms in the special case of a $k$-armed contextual bandit. One might harmlessly assume that $\NearOptimal=\IR^{kd}$, or equivalently, assuming it is the identity function if viewed as an operator, and follow the rest of this section.
\end{rem}
The next assumption demands the selected actions to be diverse in the near-optimal space. 
Specifically, recall the inverse covariance matrix  $\SymCovMatrix_{t}$ of the actions chosen by a policy from \cref{eq:sym-cov-matrix-def}.
\begin{asmp}[Linear expansion]
\label{as:linear-expansion}
We say that linear expansion holds for a policy if
\begin{align*}
    \Prob*{
        \NormOp{\NearOptimal^\top\SymCovMatrix_t\NearOptimal}
        \geq
        \frac{c_2}{t}
    }
    \leq
    \frac{c_1}{2t^2}
\end{align*}
for some constants $c_1,c_2>0$ and all $t\in[T]$ with $t\geq t_0$. We denote the indicator variable for the event $\NormOp{\NearOptimal^\top\SymCovMatrix_t\NearOptimal}
<
c_2/t$ by $\LinExpansion_t$.
\end{asmp}
We will show in \cref{lem:ROFUL-satisfies-linear-expansions} that ROFUL satisfies the linear expansion assumption, under a variant of the optimism-in-expectation assumption from \cref{sec:optimism} as well as a certain diversity assumption.
This fact, combined with the following lemma, leads to our main result of this section which is presented as \cref{cor:final-poly-log-result-on-roful}. The next lemma operates on the same setting as \cref{thm:general-regret} with additional assumptions on the reasonableness of the worth functions, the margin condition, and the linear expansions.
\begin{lem}
\label{thm:roful-polylog-regret-bound}
Consider an uncertainty structure $\{\Uncertainty_t\}_{t\ge1}$ with uncertainty complexity $\Complexity$, gap level $\delta$, associated parameter $\GapProb*[\GapLevel]$, as well as gain-rate parameters $\GainRate*[\delta]$ and $\Deviation*[\GapLevel]$. Also, assume that worth functions of \cref{alg:roful} (policy $\PolicyROFUL$) are reasonable (\cref{def:reasonableness}), and that the margin condition (\cref{as:margin-condition}) and the linear expansion condition (\cref{,as:linear-expansion}) hold. Then, the cumulative regret of $\PolicyROFUL$ satisfies the following inequality:
\begin{align*}
    \Regret(T,\PolicyROFUL)
    \leq
		\frac{\Complexity{}}{\GapLevel\GainRate*[\GapLevel]}
+
\frac{\Deviation*[\GapLevel]}{\GapLevel}
+
\GapLevel(t_0\GapProb*[\GapLevel]+c_1+1)
+
16\ArmBound^2\Radius^2c_2c_0\log(T)\,.
\end{align*}\end{lem}
The proof of \cref{thm:roful-polylog-regret-bound} is given in \cref{subsec:pf-lemma-roful-polylog-regret-bound}.

\paragraph{Linear expansion and ROFUL.} In what follows, we will show that under a certain \emph{diversity condition} (\cref{as:diversity-condition}) and a generalization of the optimism assumption in \cref{sec:optimism}, the ROFUL algorithm satisfies linear expansion.
\begin{asmp}[Diversity condition]
\label{as:diversity-condition}
We say that a GLB problem satisfies the diversity condition with parameter $\DivParam$ if $\ArmSet_t$ is independent of $\SigmaAlgebra(\ArmSet_1,\ChosenArm_1,\Reward(\ChosenArm_1),\ldots,\ArmSet_{t-1},\ChosenArm_{t-1},\Reward(\ChosenArm_{t-1}))$ and
\begin{align}
    \lambda_{\min{}}\left(\Expect*{\DivMatrix_t}\right)\geq\DivParam
    ~~~~\text{and}~~~~
    \NormOp*{\Expect*{\left(\DivMatrix_t-\Expect*{\DivMatrix_t}\right)^2}}\leq\DivCovOp^2
    ~~~~\text{for all $t\in[T]$,}
    \label{eq:diversity-condition}
\end{align}
where $\DivMatrix_t:=\NearOptimal^\top\OptimalArm_t{\OptimalArm_t}^\top\NearOptimal\cdot\IGap_t=\OptimalArm_t{\OptimalArm_t}^\top\cdot\IGap_t$.
\end{asmp}
\cref{as:diversity-condition} is similar to Assumption A3 of \cite{goldenshluger2013linear} and Assumption 4 of \cite{bastani2020online}.
\begin{defn}[Optimism in probability]
\label{def:prob-optimism}
We say that the worth function $\Worth_t(\cdot)$ is \emph{optimistic in probability} if for some $\OptCoef$ and $\OptProb$ in $(0,1]$ we have
\begin{align}
    \Prob*{
    \Big(\MeanReward_t(\OptimalArm_t)-\BaseLevel_t\Big)\cdot\ITypicalM_t
    \leq
    \frac1\OptCoef\Big(\Worth_t(\ChosenArm_t)-\BaseLevel_t\Big)\cdot\ITypicalW_t
    \Given
    \HistoryPlus_{t-1}}
    \geq
    \frac{\OptProb}{\OptCoef^2},
    \label{eq:prob-optimism-def}
\end{align}
almost surely.
\end{defn}
\begin{rem}
A slightly stronger version of \cref{eq:prob-optimism-def} is that
\begin{align}
    \Prob*{
    \Big(\MeanReward_t(\OptimalArm_t)-\BaseLevel_t\Big)\cdot\ITypicalM_t
    \leq
    \frac1\OptCoef\Big(\Worth_t(\ChosenArm_t)-\BaseLevel_t\Big)\cdot\ITypicalW_t
    \Given
    \Param,\HistoryPlus_{t-1}}
    \geq
    \frac{\OptProb}{\OptCoef^2},
    \label{eq:strong-prob-optimism-def}
\end{align}
almost surely. It is worth noting that this stronger condition also implies optimism in expectation (\cref{def:optimism-in-exp}). First, note that $\OptCoef\big(\MeanReward_t(\OptimalArm_t)-\BaseLevel_t\big)\cdot\ITypicalM_t$ is a deterministic function of $(\Param,\HistoryPlus_{t-1})$. Therefore, we have
\begin{align*}
    \frac{\OptProb}{\OptCoef^2}\cdot\OptCoef^2\Big(\MeanReward_t(\OptimalArm_t)-\BaseLevel_t\Big)^2\cdot\ITypicalM_t
    &\leq
    \Expect*{\Big(\Worth_t(\ChosenArm_t)-\BaseLevel_t\Big)^2\cdot\ITypicalW_t\Given\Param,\HistoryPlus_{t-1}}.
\end{align*}
This proves optimism in expectation since,
\begin{align*}
    \Expect*{\big(\MeanReward_t(\OptimalArm_t)-\BaseLevel_t\big)^2\cdot\ITypicalM_t}
    &=
    \frac1\OptProb\Expect*{\OptProb\big(\MeanReward_t(\OptimalArm_t)-\BaseLevel_t\big)^2\cdot\ITypicalM_t}\\
    &\leq
    \frac1\OptProb\Expect*{\Expect*{\big(\Worth_t(\ChosenArm_t)-\BaseLevel_t\big)^2\cdot\ITypicalW_t\Given\Param,\HistoryPlus_{t-1}}}\\
    &=
    \frac1\OptProb\Expect*{\big(\Worth_t(\ChosenArm_t)-\BaseLevel_t\big)^2\cdot\ITypicalW_t}\,.
\end{align*}
It is worthwhile to mention that in the worst-case analysis of an algorithm, $\Param$ is a deterministic constant and, therefore, \cref{eq:prob-optimism-def} and \cref{eq:strong-prob-optimism-def} are equivalent. Nevertheless, the stronger condition \cref{eq:strong-prob-optimism-def} need not hold when $\Param$ is drawn from a prior distribution. An example of this situation is the Bayesian analysis of TS in which \cref{eq:exp-optimism-def} and \cref{eq:prob-optimism-def} hold simultaneously, although \cref{eq:strong-prob-optimism-def} fails to hold.
\end{rem}
Now, we are ready to state our result that ROFUL satisfies the linear expansion assumption if its worth functions are optimistic in probability.
\begin{lem}[ROFUL satisfies linear expansion]\label{lem:ROFUL-satisfies-linear-expansions}
If the diversity condition (\cref{as:diversity-condition}) holds and the worth functions of ROFUL are optimistic in probability (\cref{def:prob-optimism}), then
ROFUL satisfies the linear expansion condition (\cref{as:linear-expansion}) with
\[
    c_1
    :=
    6~,~~~
    c_2
    :=
    \frac{1}{4\OptProb\DivParam}~,~~~
    t_0
    :=
    \max\Big\{k,d,t'_0,
    3t''_0\log(t''_0)
    \Big\},
\]
where
\[
    t'_0
    :=
    \frac{16\ArmBound^2{\Radius}^2kd}{\OptProb\DivParam\OptCoef^2\GapLevel^2}\cdot\log\left(\lambda+\frac{T\ArmBound^2}{d}\right)~,~~~
    t''_0
    =
    \frac{32\DivCovOp^2+16\DivParam\ArmBound^2/3}{\OptProb\DivParam^2}\,.
\]
\end{lem}
The proof of \cref{lem:ROFUL-satisfies-linear-expansions} is given in \cref{subsec:pf-lemma-ROFUL-satisfies-linear-expansions}.

Next, we state the main result of this section, which directly follows from \cref{thm:roful-polylog-regret-bound} and \cref{lem:ROFUL-satisfies-linear-expansions}.
\begin{cor}\label{cor:final-poly-log-result-on-roful}
	Consider a GLB problem that satisfies the margin condition  (\cref{as:margin-condition}) and the diversity condition (\cref{as:diversity-condition}). Also, assume that the worth functions of \cref{alg:roful} (denoted by policy $\PolicyROFUL$) are reasonable  (\cref{def:reasonableness}) and optimistic in probability (\cref{def:prob-optimism}). Then the cumulative regret of $\PolicyROFUL$ satisfies the following inequality:
\begin{align*}
	\Regret(T,\PolicyROFUL)
	\leq
	\frac{\Complexity{}}{\GapLevel\GainRate*[\GapLevel]}
	+
	\Deviation*[\GapLevel]
	+
	\GapLevel(t_0\GapProb*[\GapLevel]+c_2+1)
	+
	16\ArmBound^2\Radius^2c_2\log(T)\,,
\end{align*}
where constants $c_1$, $c_2$, and $t_0$ are defined as in \cref{lem:ROFUL-satisfies-linear-expansions}.
\end{cor}
\begin{rem}
	Note that in terms of dependence in $T$,  by \cref{cor:final-poly-log-result-on-roful} we prove a regret bound that is
	$\Order[]{\log^2(T)}$  under similar conditions as the ones by \cite{goldenshluger2013linear} \cite{bastani2020online}. Since OFUL and TS are special cases of ROFUL, this immediately provides an $\Order[]{\log^2(T)}$ regret bound for OFUL and TS as well. To the best of our knowledge, these results are new.
\end{rem}
\begin{rem}
	\cref{cor:final-poly-log-result-on-roful}  also holds when a more general \emph{$\GeneralMargin$-margin condition} for $\GeneralMargin>0$, as by \cite{goldenshluger2009woodroofe} \cite{bastani2017mostly} that replaces \cref{eq:margin-condition} with
	\begin{align}
		\Prob*{\Gap_t\leq z}\leq c_0z^\GeneralMargin
		\label{eq:alpha-margin-condition}
	\end{align}
	is satisfied. In this case, the term $16\ArmBound^2\Radius^2c_2\log(T)$ in the regret bound would be replaced by
	a term of order
	\[
	\Order*{(\ArmBound\Radius)^{\GeneralMargin+1}\left(1+\int_{1}^{T}t^{-\frac{\GeneralMargin+1}{2}}dt\right)}\,,
	\]
	through the same proof technique.
\end{rem}	

\ACKNOWLEDGMENT{This work was supported by the Stanford Data Science Initiative, and by National Science Foundation CAREER award CMMI: 1554140.}

\bibliography{papers,mypapers,books}

\begin{APPENDICES}
\section{Additional Proofs}
\label{sec:auxi}

\subsection{Proof of \cref{eq:bayesian-finite-complexity}}
\begin{proof}
	Noting that
	\begin{align*}
		\Expect{\Reward(\Arm_i)\Given\History_{t-1},\OptimalArm=\Arm_j}
		&=
		\Inner{\mu_{t,j},\Arm_i}
		~~~~~~~\text{and}~~~~~~~
		\Expect{\Reward(\Arm_i)\Given\History_{t-1}}
		=
		\Inner{\mu_{t},\Arm_i},
	\end{align*}
	we get, by Lemma 3 of \cite{russo2016information}, that for all $i\in[k]$,
	\begin{align*}
		\Arm_i^\top(\mu_{t,j}-\mu_{t})(\mu_{t,j}-\mu_{t})^\top\Arm_i&=
		\Inner{\mu_{t,j}-\mu_{t},\Arm_i}^2\\
		&\leq
		2\sigma^2\KlDiv*{
			\Prob{\Reward(\Arm_i)\Given\History_{t-1},\OptimalArm=\Arm_j}
		}{
			\Prob{\Reward(\Arm_i)\Given\History_{t-1}}
		}.
	\end{align*}
	This in turn implies that
	\begin{align*}
		\Uncertainty_t(\Arm_i)
		&\leq
		2\sigma^2
		\sum_{j=1}^{k}
		\alpha_j
		\KlDiv*{
			\Prob{\Reward(\Arm_i)\Given\History_{t-1},\OptimalArm=\Arm_j}
		}{
			\Prob{\Reward(\Arm_i)\Given\History_{t-1}}
		}\\
		&=
		2\sigma^2
		I\ParenDelim{\OptimalArm;\Reward(\Arm_i)\Given\History_{t-1}}.
	\end{align*}
	For any policy $\Policy\in\PolicySet$, we have
	\begin{align*}
		\Expect{\Uncertainty_t(\ChosenArm_t)}
		&=
		\Expect*{\Expect*{\Uncertainty_t(\ChosenArm_t)\Given\History_{t-1},\ChosenArm_t}}\\
		&=
		\Expect*{\sum_{i=1}^{k}\Prob{\ChosenArm_t=\Arm_i\Given\History_{t-1}}\cdot\Uncertainty_t(\Arm_i)}\\
		&\leq
		2\sigma^2\Expect*{\sum_{i=1}^{k}\Prob{\ChosenArm_t=\Arm_i\Given\History_{t-1}}\cdot I(\OptimalArm;\Reward(\Arm_i)\GivenSymbol\History_{t-1})}.
	\end{align*}
	Using the assumption that $\ChosenArm_t$ is independent of $\OptimalArm$ conditional on $\History_{t-1}$, we can write
	\begin{align*}
		\Expect{\Uncertainty_t(\ChosenArm_t)}
		&\leq
		2\sigma^2\Expect*{\sum_{i=1}^{k}\Prob{\ChosenArm_t=\Arm_i\Given\History_{t-1}}\cdot I(\OptimalArm;\Reward(\ChosenArm_t)\GivenSymbol[\big]\History_{t-1},\ChosenArm_t=\Arm_i)}\\
		&=
		2\sigma^2\Expect*{ I(\OptimalArm;\Reward(\ChosenArm_t)\GivenSymbol\History_{t-1},\ChosenArm_t)}\\
		&\leq
		2\sigma^2\Expect*{ I(\OptimalArm;\Reward(\ChosenArm_t)\GivenSymbol\History_{t-1},\ChosenArm_t)
			+
			I(\OptimalArm;\ChosenArm_t\GivenSymbol\History_{t-1})}\\
		&\leq
		2\sigma^2\Expect*{ I(\OptimalArm;(\Reward(\ChosenArm_t),\ChosenArm_t)\GivenSymbol[\big]\History_{t-1})
		}.
	\end{align*}
	Therefore, by summing up both sides of the above inequalities, we get
	\begin{align*}
		\Uncertainty(\Policy)
		&=
		\Expect*{\sum_{t=1}^{T}\Uncertainty_t(\ChosenArm_t)}\\
		&\leq
		2\sigma^2\Expect*{
			\sum_{t=1}^{T}I(\OptimalArm;(\Reward(\ChosenArm_t),\ChosenArm_t)\GivenSymbol\History_{t-1})
		}\\
		&\leq
		2\sigma^2\sum_{t=1}^{T}\Expect*{
			I(\OptimalArm;(\Reward(\ChosenArm_t),\ChosenArm_t)\GivenSymbol\History_{t-1})
		}\\
		&=
		2\sigma^2\Entropy(\OptimalArm).
	\end{align*}
\end{proof}

\subsection{Proof of \cref{lem:optimism-for-ts}}

First, we state the following lemma.
\begin{lem}
\label{lem:chi-squared-tail}
If $X\sim\chi_d^2$, then for all positive constants $\gamma$, we have
$\Prob*{X\geq 2\Dim+3\gamma} \le\exp\left(-\gamma\right)$.
\end{lem}
\begin{proof}
The proof follows directly from applying Lemma 1 of \cite{laurent2000adaptive} which gives
\begin{align*}
	\Prob*{X\geq 2\Dim+3\gamma}
	&\leq
	\Prob*{X\geq \Dim+2\sqrt{\Dim \gamma}+2\gamma}\\
	&\leq
	\exp\left(-\gamma\right)\,.
\end{align*}
\end{proof}
\begin{proof}[Proof of \cref{lem:optimism-for-ts}]
First, assume that $6\log 2nT\ge 2d+12\log T$.
Since $\Inflation^{-1}\SymCovMatrix_t^{-\frac12}(\TsSample_t-\Estimator_t)\sim\Normal{0,\Eye_d}$, it follows from \cref{lem:chi-squared-tail} (with $\gamma =4\log T$) that
\begin{align*}
    \Prob*{
        \Norm{\TsSample_t-\Estimator_t}_{\SymCovMatrix_t^{-1}}^2
        \geq
        \Inflation^2[2\Dim+12\log T]
    }
    \leq
    {1}/{T^4}
    \leq
    {1}/{(2T^3)}\,.
\end{align*}
Therefore, combining this with $\log 2nT\ge 2d+12\log T$, we have
\begin{align*}
    \Prob*{\forall\Arm\in\ArmSet_t:\Worth_t(\Arm)\in[\LowerConf_t(\Arm),\UpperConf_t(\Arm)]}
    &=
    \Prob*{\forall\Arm\in\ArmSet_t:\Abs{\Inner[\big]{\TsSample_t-\Estimator_t,\Arm}}\leq\Radius'\Norm{\Arm}_{\SymCovMatrix_t}}\\
    &\geq
    \Prob*{
        \Norm[]{\TsSample_t-\Estimator_t}_{\SymCovMatrix_t^{-1}}^2
        \leq
        \Inflation^2(2\Dim+12\log T)
    }\\
&\geq
    1-    {1}/{(2T^3)}\,.
\end{align*}
In the finite action set case, we provide a different bound using the union bound. For each $\Arm\in\ArmSet_t$, note that $\Inner[]{\TsSample_t-\Estimator_t,\Arm}\sim\Normal{0,\Inflation^2\Norm[]{\Arm}_{\SymCovMatrix_t}^2}$. Hence, we have,
\begin{align*}
    \Prob*{\forall\Arm\in\ArmSet_t:\Worth_t(\Arm)\in[\LowerConf_t(\Arm),\UpperConf_t(\Arm)]}
    &=
    \Prob*{\forall\Arm\in\ArmSet_t:\Abs[]{\Inner[]{\TsSample_t-\Estimator_t,\Arm}}\leq\Radius'\Norm{\Arm}_{\SymCovMatrix_t}}\\
    &\geq
    1-n\cdot
    \Phi\left(-\frac{\Radius'}{\Inflation}\right)\geq
    1-    {1}/{(2T^3)}\,,
\end{align*}
where in the last step we used the fact that $\Phi(-x)\le \exp(-x^2/2)/(x\sqrt{2\pi})$ for all positive $x$.
\end{proof}

\subsection{Proof of  \cref{thm:roful-polylog-regret-bound}}\label{subsec:pf-lemma-roful-polylog-regret-bound}

\begin{proof}
	The main idea is to refine the proof of \cref{thm:general-regret}. We first recall \cref{eq:mid-pf-thm-general-regret-to-use-for-poly-log-results}:
	\begin{align}
		\Regret(T,\PolicyROFUL)
		&\leq
		\frac{\Complexity{}}{\GapLevel\GainRate*[\GapLevel]}
		+
		\frac{\Deviation*[\GapLevel]}{\GapLevel}
		+
		\sum_{t=1}^{T}
		\Expect*{
			\Big(\MeanReward_t(\OptimalArm_t)-\MeanReward_t(\ChosenArm_t)\Big)
			\cdot
			(1-\BigRegret_t)
		}.
		\label{eq:regret-refinement-ineq}
	\end{align}
	where $\BigRegret_t$ denotes $\II(\MeanReward_t(\OptimalArm_t)-\MeanReward_t(\ChosenArm_t)\geq\GapLevel)$. We next improve the upper bound for  each individual term in the above sum. For $t\leq t_0$, where $t_0$ is defined as in \cref{as:linear-expansion}, we use our previous bound
	\begin{align*}
		\Expect*{
			\Big(\MeanReward_t(\OptimalArm_t)-\MeanReward_t(\ChosenArm_t)\Big)
			\cdot
			(1-\BigRegret_t)
		}
		\leq
		\GapLevel\GapProb*[\GapLevel]\,.
	\end{align*}
	Next, we consider $t>t_0$. Whenever $\BigRegret_t=0$, we have $\MeanReward_t(\ChosenArm_t)>\MeanReward_t(\OptimalArm_t)-\GapLevel$, which in turn implies that $\OptimalArm_t,\ChosenArm_t\in\NearOptimal$.
	By recalling the indicator variable $\ITypical_t$ defined in the proof of \cref{thm:roful-gain-rate}, and provided that $(1-\BigRegret_t)\LinExpansion_t\ITypical_t=1$, we have
	\begin{align*}
		\MeanReward_t(\OptimalArm_t)-\MeanReward_t(\ChosenArm_t)
		&\overset{(a)}{\leq}
		\UpperConf_t(\OptimalArm_t)-\LowerConf_t(\ChosenArm_t)\\
		&\leq
		\LowerConf_t(\OptimalArm_t)-\UpperConf_t(\ChosenArm_t)+2\Radius\Big(\Norm[\big]{\OptimalArm_t}_{\SymCovMatrix_t}+\Norm[\big]{\ChosenArm_t}_{\SymCovMatrix_t}\Big)\\
		&\overset{(b)}{\leq}
		\Worth_t(\OptimalArm_t)-\Worth_t(\ChosenArm_t)+2\Radius\Big(\Norm[\big]{\OptimalArm_t}_{\SymCovMatrix_t}+\Norm[\big]{\ChosenArm_t}_{\SymCovMatrix_t}\Big)\\
		&\overset{(c)}{\leq}
		2\Radius\Big(\Norm[\big]{\OptimalArm_t}_{\SymCovMatrix_t}+\Norm[\big]{\ChosenArm_t}_{\SymCovMatrix_t}\Big)\\
		&\overset{(d)}{\leq}
		\frac{4\ArmBound\Radius\sqrt{c_2}}{\sqrt t}.
	\end{align*}
	In the above, (a) holds since $\ITypicalM_t=1$, (b) follows from $\ITypicalW_t=1$, (c) uses the fact that ROFUL chooses the action with maximum worth $\Worth_t(\cdot)$, and (d) is a consequence of $\LinExpansion_t=1$. Now, using this inequality, we can write
	\begin{align*}
		\Expect*{
			\left(\MeanReward_t(\OptimalArm_t)-\MeanReward_t(\ChosenArm_t)\right)
			\cdot
			(1-\BigRegret_t)
		}
		&=
		\Expect*{
			\left(\MeanReward_t(\OptimalArm_t)-\MeanReward_t(\ChosenArm_t)\right)
			\cdot
			(1-\BigRegret_t)(1-\LinExpansion_t\ITypical_t+\LinExpansion_t\ITypical_t)
		}\\
		&\leq
		\Expect*{
			\left(\MeanReward_t(\OptimalArm_t)-\MeanReward_t(\ChosenArm_t)\right)
			\cdot
			(1-\BigRegret_t)\LinExpansion_t\ITypical_t
		}
		+
		\GapLevel\,\Prob*{\LinExpansion_t\ITypical_t=0}\\
		&\leq
		\frac{4\ArmBound\Radius\sqrt{c_2}}{\sqrt t}
		\Prob*{\LinExpansion_t\ITypical_t=1,\,\Gap_t<\frac{4\ArmBound c_2\Radius}{\sqrt t}}
		+
		\frac{\GapLevel(c_1+1)}{2t^2}\\
		&\leq
		\frac{4\ArmBound\Radius\sqrt{c_2}}{\sqrt t}
		\Prob*{\Gap_t<\frac{4\ArmBound c_2\Radius}{\sqrt t}}
		+
		\frac{\GapLevel(c_1+1)}{2t^2}\\
		&\leq
		\frac{16\ArmBound^2\Radius^2c_2c_0}{t}
		+
		\frac{\GapLevel(c_1+1)}{2t^2}\,,
	\end{align*}
where the last step uses \cref{as:margin-condition}.
	This inequality in combination with \cref{eq:regret-refinement-ineq} yields
	\begin{align*}
		\Regret(T,\PolicyROFUL)
		&\leq
		\frac{\Complexity{}}{\GapLevel\GainRate*[\GapLevel]}
		+
		\frac{\Deviation*[\GapLevel]}{\GapLevel}
		+
		\GapLevel t_0\GapProb*[\GapLevel]
		+
		\sum_{t=t_0+1}^{T}
		\left\{\frac{16\ArmBound^2\Radius^2c_2c_0}{t}
		+
		\frac{\GapLevel(c_1+1)}{2t^2}\right\}\\
		&\leq
		\frac{\Complexity{}}{\GapLevel\GainRate*[\GapLevel]}
		+
		\frac{\Deviation*[\GapLevel]}{\GapLevel}
		+
		\GapLevel(t_0\GapProb*[\GapLevel]+c_1+1)
		+
		16\ArmBound^2\Radius^2c_2c_0\log(T)\,,
	\end{align*}
	which is the desired result.
\end{proof}

\subsection{Proof of \cref{lem:ROFUL-satisfies-linear-expansions}}\label{subsec:pf-lemma-ROFUL-satisfies-linear-expansions}

\begin{proof}
	For any $t\in[T]$, let $\IOptim_t$ be a Bernoulli random variable with
	\[
	\Prob*{\IOptim_t=1\Given\HistoryPlus_{t-1}}=\frac{\OptProb}{\OptCoef^2}
	\]
	almost surely such that
	\begin{align}
		\IOptim_t=1~~~~\text{and}~~~~\ITypical_t=1
		\implies
		\MeanReward_t(\OptimalArm_t)-\BaseLevel_t
		\leq
		\frac1\OptCoef
		\Big(\Worth_t(\ChosenArm_t)-\BaseLevel_t\Big)\,.
		\label{eq:optim-implication}
	\end{align}
	The existence of this random variable is guaranteed by the optimism-in-probability assumption. Next, for $t\in[T]$, we have
	\begin{align*}
		\sum_{i=\frac t2}^{t}\ChosenArm_i\ChosenArm_i^\top\cdot(1-\BigRegret_i)\IGap_i\IOptim_i
		\semdefleq
		\SymCovMatrix_t^{-1}.
	\end{align*}
	Moreover, it follows from the definition of $\BigRegret_i$ and $\IGap_i$ that
	\begin{align*}
		\ChosenArm_i\ChosenArm_i^\top\cdot(1-\BigRegret_i)\IGap_i
		=
		{\OptimalArm_i}{\OptimalArm_i}^\top\cdot(1-\BigRegret_i)\IGap_i
	\end{align*}
	for all $i\in[T]$. Therefore, we get 
	\begin{align*}
		\sum_{i=\frac t2}^{t}{\OptimalArm_i}{\OptimalArm_i}^\top\cdot(1-\BigRegret_i)\IGap_i\IOptim_i
		\semdefleq
		\SymCovMatrix_t^{-1}.
	\end{align*}
	By recalling $\DivMatrix_i=\OptimalArm_i{\OptimalArm_i}^\top\cdot\IGap_i$, we have
	\begin{align*}
		\NormOp{\NearOptimal^\top\SymCovMatrix_t\NearOptimal}
		&=
		\lambda_{\min{}}\left(\NearOptimal^\top\SymCovMatrix_t^{-1}\NearOptimal\right)\\
		&\geq
		\lambda_{\min{}}\left(\sum_{i=\frac t2}^{t}\DivMatrix_i\cdot(1-\BigRegret_i)\IOptim_i\right)\\
		&\geq
		\lambda_{\min{}}\left(\sum_{i=\frac t2}^{t}\DivMatrix_i\cdot\IOptim_i\right)
		-
		\lambda_{\max{}}\left(\sum_{i=\frac t2}^{t}\DivMatrix_i\cdot\BigRegret_i\IOptim_i\right).
	\end{align*}
	We now bound each term separately. 
Next, we prove that the smallest singular value of $\sum_{i=\frac t2}^{t}\DivMatrix_i\cdot\IOptim_i$ grows linearly with high probability. Using the noncommutative Bernstein's inequality for $s\geq0$, (e.g., Theorem 1.4 of \cite{tropp2012user}), we get
	\begin{align*}
		\Prob*{
			\NormOp[\bigg]{\sum_{i=\frac t2}^{t}\left(\DivMatrix_i\cdot\IOptim_i-\Expect*{\DivMatrix_i\cdot\IOptim_i}\right)}
			\geq
			s
		}
		\leq
		kd\cdot\exp\left(-\frac{s^2/2}{t\OptProb\DivCovOp^2+s\ArmBound^2/3}\right).
	\end{align*}
	Setting $s:={t\OptProb\DivParam}/{2}$ and applying the triangle inequality yields
	\begin{align}
		\Prob*{
			\NormOp[\bigg]{\sum_{i=\frac t2}^{t}\DivMatrix_i\cdot\IOptim_i}
			\leq
			\frac{t\OptProb\DivParam}{2}
		}
		\leq
		kd\cdot\exp\left(-\frac{t\OptProb\DivParam^2}{8\DivCovOp^2+4\DivParam\ArmBound^2/3}\right).
		\label{eq:gamma-optimism-bound}
	\end{align}
	Our next goal is to prove an upper bound for the largest singular value of $\sum_{i=\frac t2}^{t}\DivMatrix_i\cdot\BigRegret_i\IOptim_i$. We apply the following bound:
	\begin{align*}
		\lambda_{\max{}}\left(\sum_{i=\frac t2}^{t}\DivMatrix_i\cdot\BigRegret_i\IOptim_i\right)
		&\leq
		\ArmBound^2\sum_{i=\frac t2}^{t}\BigRegret_i\IOptim_i.
	\end{align*}
	Using \cref{eq:optim-implication}, we can deduce that, whenever $\ITypical_t=1$ and $\BigRegret_i\IOptim_i=1$, we have
	\begin{align*}
		\GapLevel
		&\leq
		\MeanReward_i(\OptimalArm_i)-\Inner[\big]{\Param,\ChosenArm_i}\\
		&\leq
		\MeanReward_i(\OptimalArm_i)-\LowerConf_i(\ChosenArm_i)\\
		&=
		\MeanReward_i(\OptimalArm_i)-\BaseLevel_i+\BaseLevel_i-\LowerConf_i(\ChosenArm_i)\\
		&\leq
		\frac1\OptCoef(\Worth_i(\ChosenArm_i)-\BaseLevel_i)+\BaseLevel_i-\LowerConf_i(\ChosenArm_i)\\
		&\leq
		\frac1\OptCoef(\Worth_i(\ChosenArm_i)-\BaseLevel_i+\BaseLevel_i-\LowerConf_i(\ChosenArm_i))\\
		&=
		\frac1\OptCoef\left(\Worth_i(\ChosenArm_i)-\LowerConf_i(\ChosenArm_i)\right)\\
		&\leq
		\frac1\OptCoef\left(\UpperConf_i(\ChosenArm_i)-\LowerConf_i(\ChosenArm_i)\right)\\
		&\leq
		\frac{2\Radius}\OptCoef\Norm[\big]{\ChosenArm_i}_{\SymCovMatrix_{i-1}}.
	\end{align*}
	Therefore, we can write
	\begin{align*}
		\ITypical_i\BigRegret_i\IOptim_i
		&\leq
		\left(\frac{2\Radius}{\OptCoef\GapLevel}\cdot\Norm[\big]{\ChosenArm_i}_{\SymCovMatrix_{i-1}}\right)^2
		<
		\frac{4{\Radius}^2}{\OptCoef^2\GapLevel^2}\cdot\Norm[\big]{\ChosenArm_i}_{\SymCovMatrix_{i-1}}^2.
	\end{align*}
	Next, Lemma 10 and Lemma 11 in \cite{abbasi2011improved} yield
	\begin{align*}
		\left(\prod_{i=\frac t2}^{t}\ITypical_i\right)\sum_{i=\frac t2}^{t}\BigRegret_i\IOptim_i
		&\leq
		\sum_{i=\frac t2}^{t}\ITypical_i\BigRegret_i\IOptim_i\\
		&\leq
		\sum_{i=1}^{t}\ITypical_i\BigRegret_i\IOptim_i\\
		&\leq
		\frac{4{\Radius}^2}{\OptCoef^2\GapLevel^2}\cdot\sum_{i=1}^{t}\Norm[\big]{\ChosenArm_i}_{\SymCovMatrix_{i-1}}^2\\
		&\leq
		\frac{4{\Radius}^2kd}{\OptCoef^2\GapLevel^2}\cdot\log\left(\lambda+\frac{t\ArmBound^2}{d}\right).
	\end{align*}
	Hence, it is a direct consequence of \cref{eq:gamma-optimism-bound} that, for any
	$
	t
	\geq
	t_0
	\geq
	\frac{16\ArmBound^2{\Radius}^2kd}{\OptProb\DivParam\OptCoef^2\GapLevel^2}\cdot\log\left(\lambda+\frac{T\ArmBound^2}{d}\right)
	$, we get
	\begin{align}
		\Prob*{
			\NormOp{\NearOptimal^\top\SymCovMatrix_t\NearOptimal}
			\geq
			\frac{4}{t\OptProb\DivParam}
		}
		&\leq
		kd\cdot\exp\left(-\frac{t\OptProb\DivParam^2}{8\DivCovOp^2+4\DivParam\ArmBound^2/3}\right)
		+
		\Prob*{\prod_{i=\frac t2}^{t}=0}\\
		&\leq
		kd\cdot\exp\left(-\frac{t\OptProb\DivParam^2}{8\DivCovOp^2+4\DivParam\ArmBound^2/3}\right)
		+
		\frac{2}{t^2}.
	\end{align}
	We prove that for sufficiently large $t$, the right-hand side of the above inequality is bounded above by $3/t^2$. This is equivalent to
	\begin{align*}
		\frac{8\DivCovOp^2+4\DivParam\ArmBound^2/3}{\OptProb\DivParam^2}
		\leq
		\frac{t}{\log(t^2kd)}
		=
		\frac{t}{2\log(t)+\log(kd)}.
	\end{align*}
	Using \cref{lem:implicit-log-inequality} below, we infer that this is satisfied for all $t\geq t_0$.
\end{proof}

\begin{lem}
	Let $a\geq3$ be given. Then, for all $t\geq3a\log(a)$, we have
	$
	\frac{t}{\log(t)}\geq a
	$.
	\label{lem:implicit-log-inequality}
\end{lem}
\begin{proof}
	First, note that $f:t\mapsto\frac{t}{\log(t)}$ is an increasing function of $t$ for all $t\geq e$. To see this, we compute the derivative of $f$ as follows:
	\begin{align*}
		f'(t)=\frac{\log(t)-1}{\log^2(t)} \geq 0.
	\end{align*}
	Next, setting $t_0=3a\log(a)$, we have
	\begin{align*}
		f(t_0)
		&=
		\frac{3a\log(a)}{\log(3a\log(a))}\\
		&\geq
		\frac{3a\log(a)}{\log(a^3)}\\
		&=
		a.
	\end{align*}
\end{proof}

\end{APPENDICES}

\end{document}